\documentclass{article}
\usepackage{arxiv}

\usepackage{silence}

\usepackage{microtype}
\usepackage[%
pdfstartview=FitH,%
breaklinks=true,%
bookmarks=false,%
colorlinks=true,%
linkcolor= blue,
anchorcolor=blue,%
citecolor=blue,
filecolor=blue,%
menucolor=blue,%
urlcolor=blue%
]{hyperref}
\usepackage{url}
\usepackage{booktabs}
\usepackage{amsfonts}
\usepackage{nicefrac}
\usepackage{xcolor}
\usepackage{amsmath}
\usepackage{algorithm}
\usepackage{algorithmic}
\usepackage{amsthm}
\usepackage{amssymb}
\usepackage{xspace}
\usepackage{mathtools}
\usepackage{graphicx}
\usepackage{pdflscape}
\usepackage{subcaption}
\usepackage{natbib}
\usepackage{thmtools}
\usepackage{thm-restate}
\usepackage{bbm}
\usepackage{multirow}
\usepackage{enumitem}

\usepackage{cleveref}
\usepackage{crossreftools}
\newcounter{algline}

\crefname{algline}{line}{lines}
\Crefname{algline}{Line}{Lines}
  {\setcounter{algline}{0}\begin{algorithmic}[#1]}%
  {\end{algorithmic}}

\definecolor{blue}{HTML}{0077BB}
\definecolor{cyan}{HTML}{33BBEE}
\definecolor{green}{HTML}{009988}
\definecolor{orange}{HTML}{EE7733}
\definecolor{red}{HTML}{CC3311}
\definecolor{magenta}{HTML}{EE3377}
\definecolor{grey}{HTML}{BBBBBB}

\usepackage{amsmath}
\usepackage{xspace}

\newcommand{\h}{h}

\newcommand{\ie}{\textit{i.e.}\xspace}

\newcommand{\eg}{\textit{e.g.}\xspace}

\newcommand{\iid}{\textit{i.i.d.}\xspace}

\newcommand{\sign}{\operatorname{sign}}

\newcommand{\defeq}{:=}

\newcommand{\KL}{\operatorname{KL}}

\newcommand{\Risk}{\mathrm{R}}

\newcommand{\Riskemp}{\widehat{\mathrm{R}}}

\newcommand{\Beta}{\textrm{Beta}}

\newcommand{\MajVote}[1]{H_{#1}}
\newcommand{\prior}{\pi}
\newcommand{\post}{\rho}
\newcommand{\posterior}{\post}

\newcommand{\DgOne}{\Dcal_{|\GrOne}}
\newcommand{\DgTwo}{\Dcal_{|\GrTwo}}

\newcommand{\Sg}{S_{|g}}

\newcommand{\bound}{\Risk}
\newcommand{\BoundMV}{\mathfrak{R}}
\newcommand{\objG}{\mathbf{G}}
\newcommand{\objMV}{\mathbf{H}}

\newcommand{\AttrSpace}{\Gcal}
\newcommand{\GrOne}{a}
\newcommand{\GrTwo}{b}

\newcommand{\klup}{\overline{\kl}}
\newcommand{\kldown}{\underline{\kl}}

\newcommand{\RF}{\mathrm{RF}}
\newcommand{\RFemp}{\widehat{\mathrm{RF}}}
\newcommand{\FairRisk}[2]{\RF_{#1}\!\left(#2\right)}

\newcommand{\FairRiskemp}[2]{\RFemp_{#1}\!\left(#2\right)}

\newcommand{\DP}{\mathrm{DP}}

\newcommand{\DemParity}[2]{\DP_{#1}\!\left(#2\right)}

\newcommand{\EOP}{\mathrm{EOP}}

\newcommand{\EqualOpp}[2]{\EOP_{#1}\!\left(#2\right)}
\newcommand{\EqualOppEmp}[2]{\widehat{\EOP}_{#1}\!\left(#2\right)}

\newcommand{\EO}{\mathrm{EO}}

\newcommand{\EqualOdds}[3]{\EO_{#1}^{#3}\!\left(#2\right)}

\newcommand{\Hrisk}[1]{\Risk_{#1}\!\left(\h\right)}
\newcommand{\Hriskemp}[1]{\Riskemp_{#1}\!\left(\h\right)}

\newcommand{\Gibbs}{G}
\newcommand{\GibbsRisk}[1]{\Risk_{#1}\!\left(\Gibbs_\post\right)}

\newcommand{\GibbsRiskemp}[1]{\Riskemp_{#1}\!\left(\Gibbs_\post\right)}

\newcommand{\MVrisk}[2]{\Risk_{#1}\!\left(\MajVote{#2}\right)}

\newcommand{\GrHrisk}[2]{\Risk_{#1_{#2}}\!\left(\h\right)}
\newcommand{\GrHriskemp}[2]{\Riskemp_{#1_{#2}}\!\left(\h\right)}

\newcommand{\GrGibbsRisk}[2]{\Risk_{#1_{#2}}\!\left(G_\post\right)}
\newcommand{\GrGibbsRiskemp}[2]{\Riskemp_{#1_{#2}}\!\left(G_\post\right)}

\renewcommand{\b}{\ensuremath{\mathfrak{0}}}
\renewcommand{\c}{\ensuremath{\mathfrak{1}}}

\newcommand{\dir}{\operatorname{Dir}}
\newcommand{\cat}{\Ccal}
\newcommand{\gaus}{\Ncal}

\newcommand{\bD}{\b_{\Dcal}} 
\newcommand{\cD}{\c_{\Dcal}}
\newcommand{\bDg}[1]{\b_{\Dcal_{|#1}}} 
\newcommand{\cDg}[1]{\c_{\Dcal_{|#1}}}

\newcommand{\TripleB}{\bD^\post(\h)}
\newcommand{\TripleC}{\cD^\post(\h)}

\newcommand{\TripleBH}{\bD^\post(h_{\alphabf})}
\newcommand{\TripleCH}{\cD^\post(h_{\alphabf})}

\newcommand{\GrTripleB}[1]{\bDg{#1}^\post(\h)}
\newcommand{\GrTripleC}[1]{\cDg{#1}^\post(\h)}

\newcommand{\DownTripleB}{\bD^{\downarrow}(\h, \post)}
\newcommand{\UpTripleB}{\bD^{\uparrow}(\h, \post)}
\newcommand{\DownTripleC}{\cD^{\downarrow}(\h, \post)}
\newcommand{\UpTripleC}{\cD^{\uparrow}(\h, \post)}

\DeclareMathOperator*{\Esp}{\mathbb{E}}

\newcommand{\Ccal}{{\mathcal{C}}}
\newcommand{\Dcal}{{\mathcal{D}}}

\newcommand{\Fcal}{{\mathcal{F}}}
\newcommand{\Gcal}{{\mathcal{G}}}
\newcommand{\Hcal}{{\mathcal{H}}}

\newcommand{\Ncal}{{\mathcal{N}}}

\newcommand{\Pcal}{{\mathcal{P}}}
\newcommand{\Qcal}{{\mathcal{Q}}}

\newcommand{\Xcal}{{\mathcal{X}}}
\newcommand{\Ycal}{{\mathcal{Y}}}

\newcommand{\Pbb}{{\mathbb{P}}}

\newcommand{\Rbb}{{\mathbb{R}}}

\newcommand{\fbf}{{\mathbf f}}

\newcommand{\ibf}{{\mathbf i}}

\newcommand{\wbf}{{\mathbf w}}
\newcommand{\xbf}{{\mathbf x}}

\newcommand{\Ibf}{{\mathbf I}}

\newcommand{\kl}{{\textup{kl}}}

\newcommand{\alphabf}{{\boldsymbol{\alpha}}}

\DeclareMathOperator*{\argmax}{argmax}
\DeclareMathOperator*{\argmin}{argmin}

\newcommand{\halpha}{h_{\alphabf}}

\theoremstyle{plain}
\newtheorem{theorem}{Theorem}[section]

\theoremstyle{definition}

\newtheorem{example}[theorem]{Example}
\newtheorem{heuristic}[theorem]{Heuristic}
\theoremstyle{remark}

\crefname{definition}{Def.}{Defs.}
\Crefname{definition}{Definition}{Definitions}
\crefname{equation}{Eq.}{Eqs.}
\Crefname{equation}{Equation}{Equations}
\crefname{theorem}{Th.}{Ths.}
\Crefname{theorem}{Theorem}{Theorems}
\crefname{corollary}{Cor.}{Cors.}
\Crefname{corollary}{Corollary}{Corollaries}
\crefname{assumption}{Ass.}{Ass.}
\Crefname{assumption}{Assumption}{Assumptions}
\crefname{section}{Sec.}{Secs.}
\Crefname{section}{Section}{Sections}
\crefname{lemma}{Lem.}{Lems.}
\Crefname{lemma}{Lemma}{Lemmas}
\crefname{heuristic}{Heur.}{Heurs.}
\Crefname{heuristic}{Heuristic}{Heuristics}

\newif\ifnotappendix
\notappendixtrue

\title{PAC-Bayesian Generalization Guarantees for Fairness\\on Stochastic and Deterministic Classifiers}
\author{Julien Bastian\\
Université Jean Monnet Saint-Étienne, CNRS, Institut d Optique Graduate School,\\
Laboratoire Hubert Curien UMR 5516, F-42023, Saint-Etienne, France\\
\texttt{julien.bastian@univ-st-etienne.fr}
\AND 
Benjamin Leblanc, \hspace*{5mm} Pascal Germain\\
Département d'informatique et de génie logiciel, Université Laval, Québec, Canada\\
\texttt{benjamin.leblanc.2@ulaval.ca}\hspace*{5mm} \texttt{pascal.germain@ulaval.ca}
\AND 
Amaury Habrard \\
Université Jean Monnet Saint-Étienne, CNRS, Institut d Optique Graduate School,\\
Laboratoire Hubert Curien UMR 5516, Inria, F-42023, Saint-Etienne, France\\
Institut Universitaire de France\\
\texttt{amaury.habrard@univ-st-etienne.fr}
\AND
Christine Largeron\\
Université Jean Monnet Saint-Étienne, CNRS, Institut d Optique Graduate School,\\
Laboratoire Hubert Curien UMR 5516, F-42023, Saint-Etienne, France\\
\texttt{christine.largeron@univ-st-etienne.fr}
\AND
Guillaume Metzler\\
Université Lumière Lyon 2, Universite Claude Bernard Lyon 1, ERIC, 69007, Lyon, France \\
\texttt{guillaume.metzler@univ-lyon2.fr}
\AND 
Emilie Morvant \\
Université Jean Monnet Saint-Étienne, CNRS, Institut d Optique Graduate School,\\
Laboratoire Hubert Curien UMR 5516, F-42023, Saint-Etienne, France\\
\texttt{emilie.morvant@univ-st-etienne.fr}
\AND 
Paul Viallard\\
Univ Rennes, Inria, CNRS IRISA - UMR 6074, F35000 Rennes, France\\
\texttt{paul.viallard@inria.fr}
}

\begin{document}

\vspace*{-1cm}

\maketitle

\vspace*{-1.2cm}
\begin{abstract}
\vspace*{-.5cm}
Classical PAC generalization bounds on the prediction risk of a classifier are insufficient to provide theoretical guarantees on fairness when the goal is to learn models balancing predictive risk and fairness constraints.
We propose a PAC-Bayesian framework for deriving generalization bounds for fairness, covering both stochastic and deterministic classifiers. 
For stochastic classifiers, we derive a fairness bound using standard PAC-Bayes techniques.
Whereas for deterministic classifiers, as usual PAC-Bayes arguments do not apply directly, we leverage a recent advance in PAC-Bayes to extend the fairness bound beyond the stochastic setting. 
Our framework has two advantages: \textit{(i)} It applies to a broad class of fairness measures that can be expressed as a risk discrepancy, and \textit{(ii)} it leads to a self-bounding algorithm in which the learning procedure directly optimizes a trade-off between generalization bounds on the prediction risk and on the fairness.
We empirically evaluate our framework with three classical fairness measures, demonstrating not only its usefulness but also the tightness of our bounds.

\end{abstract}
\section{Introduction}

\looseness=-1
With the widespread use of machine learning on human-related data comes a major responsibility: Ensuring that learning algorithms do not incorporate the discriminative biases contained in the training data.
As a consequence, understanding and mitigating these biases has become a major topic~\citep[\eg,][]{czarnowska2021quantifying,mehrabi2021survey,caton2024surveyfairness}. 
In this paper, we focus on group fairness  \citep[that is, the ability to treat groups equitably without causing discrimination, see, \eg,][]{dwork2012fairness,ustun19a}, and we propose a principled approach to learn fair prediction models by leveraging generalization bounds from statistical learning theory.  
Traditionally, such generalization bounds provide high-probability guarantees that the empirical error of a model is close to its true error on unseen data. 
However, most existing generalization guarantees typically concern predictive risk only and offer no guarantee on the fairness of the model.

\looseness=-1
Our work addresses this gap by deriving generalization bounds for the family of fairness measures that can be expressed as a difference of risks. 
This yields a general framework encompassing several classical group fairness measures, such as Demographic Parity \citep{dwork2012fairness}, Equalized Odds or Equal Opportunity \citep{hardt2016equality}.
Our analysis relies on PAC-Bayesian theory \citep{ShaweTaylor1997,McAllester1998SomePT}, which offers a key advantage among learning theory tools: It enables the derivation of generalization bounds computable from data and directly optimizable, leading to self-bounding algorithms~\citep{freund1998self}.
Related works in this direction remain limited.
In particular, \citet{oneto2019,oneto2020} derived the first PAC-Bayesian generalization bound for fairness, but this theoretical result has not been computationally exploited and is restricted to stochastic classifiers and Equal Opportunity.
More broadly, existing approaches to fairness generalization \citep[\eg,][]{laakom2025fairness,woodworth2017,agarwal2018,denis2024} typically provide post-hoc guarantees: The bounds apply to a fixed model or algorithm after training and do not intervene in the learning process itself. 
To our knowledge, we propose the first PAC-Bayesian generalization bound for such general fairness measures for deterministic classifiers by leveraging recent advances in PAC-Bayes \citep{Leblanc25}. 
A key distinction with existing fairness generalization bounds is that ours lead to a self-bounding algorithm.
Indeed, the learned model comes with its own intrinsic certification of both accuracy and fairness, thereby enhancing the model's trustworthiness.

\looseness=-1
\textbf{Organization of the paper.}
\Cref{sec:supervised} covers PAC-Bayesian background for binary classification.
\Cref{sec:fairness} reviews the group fairness setting, 
and discusses existing generalization bounds.
In \Cref{section:fairness_generalization_bounds}, we derive new PAC-Bayesian bounds for fairness, which are used in \Cref{sec:algo} to design a self-bounding algorithm, empirically evaluated in \Cref{sec:expe}.

\section{General Supervised Classification Setting}

\label{sec:supervised}
\subsection{Setting and Notations} 
\looseness=-1
A supervised binary classification task is modeled by an unknown distribution $\Pcal$ defined over $\Xcal{\times}\Ycal$, where $\Xcal\!\subseteq\!\Rbb^d$ is the input space of dimension $d$ and $\Ycal\!=\! \{-1,+1\}$ is the output space ($\Pcal_{\Xcal}$, \textit{resp.} $\Pcal_{\Ycal}$, denotes the marginal distribution of $\Pcal$ over $\Xcal$, \textit{resp.} $\Ycal$).
A learning sample $S\!=\!\{(x_i,y_i)\}_{i=1}^m $ contains $m$ examples drawn \iid from~$\Pcal$; We denote by $\Pcal^m$ the distribution of such a $m$-sample.
Let~$\Hcal$ be a hypothesis space, where each $\h \!\in\! \Hcal$ is a classifier $\h:\Xcal\!\to\!\Ycal'$.
In most machine learning literature, the learner's objective is to find the hypothesis $\h$ that assigns a label $y\!\in\!\Ycal$ to an input $x$ as accurately as possible. 
Given a binary loss $\ell: \Ycal' {\times} \Ycal \!\to\! \{0,1\}$, the \textit{true risk} $\Hrisk{\Pcal}$ of a hypothesis $\h$ on the distribution $\Pcal$ is defined by
\begin{align*}
\Hrisk{\Pcal} \defeq \Esp_{(x,y)\sim\Pcal}\,\ell(\h(x),y),
\end{align*}
and the \textit{empirical risk} 
is \mbox{$\Hriskemp{S} \!\defeq\! \frac{1}{m} \sum_{i=1}^{m} \ell(\h(x_i),y_i)$}.
Then, the learner aims to find the \emph{best} hypothesis $\h\!\in\!\Hcal$ that minimizes $\Hrisk{\Pcal}$. 
Since $\Pcal$ is unknown, statistical learning theory approaches \citep[see, \eg,][]{vapnik1999nature} classically promote minimizing an upper bound on the generalization gap: $| \Hrisk{\Pcal}{-}\Hriskemp{S} |$. This line of thought is notably represented by PAC (Probably Approximately Correct) generalization bounds \citep{Valiant1984} that take the form 
\begin{align*}
\Pbb_{S\sim\Pcal^m} \!\left[
\forall \h \in \Hcal,\  
\left| \Hrisk{\Pcal}{-}\Hriskemp{S} \right| \leq \epsilon\left(\tfrac1\delta,\tfrac1m \right) \right] \!\geq\! 1{-}\delta.
\end{align*}
Put into words, with high-probability (at least $1{-}\delta$) over the random choice of the sample $S$, tight generalization guarantees are obtained when the deviation $\big|\Hrisk{\Pcal}\! -\! \Hriskemp{S}\big|$ is low, \ie, one wants  $\epsilon(\frac1\delta,\frac1m)$ to be as small as possible.

Our work studies a specific family of generalization bounds known as PAC-Bayesian bounds, which have the advantage of providing tight bounds that can be estimable, allowing us to develop algorithms looking for the direct minimization of the bound.
Such algorithms are referred to as self-bounding algorithms \citep{freund1998self}.

\subsection{Classical PAC-Bayesian Theory}\label{sec:pac_bayesian_theory}
\looseness=-1
PAC-Bayesian theory, introduced by \citet{McAllester1998SomePT,ShaweTaylor1997}, provides generalization bounds in expectation over the hypothesis space~$\Hcal$.
It assumes a \textit{prior} distribution $\prior$ over~$\Hcal$, encoding an \textit{apriori} belief about the hypotheses before observing the learning sample~$S$. 
Then, given $S$, $\prior$, and $\Hcal$, the learner outputs a \textit{posterior} distribution $\post$ over $\Hcal$.
This framework studies the risk of the stochastic \textit{Gibbs classifier}, denoted by~$\Gibbs_{\post}$, which classifies an input $x\!\in\!\Xcal$ by first sampling a hypothesis~$\h$ from~$\posterior$ and then outputting $\h(x)$.
Accordingly, the \textit{Gibbs risk} (true risk) is defined by 
\begin{align}
\label{eq:gibbs}
\GibbsRisk{\Pcal} \defeq \Esp_{\h\sim \post} \Hrisk{\Pcal} =
\Esp_{\h\sim \post} \Esp_{(x,y) \sim \Pcal}  \ell(\h(x),y).
\end{align}
 Its empirical counterpart is  \mbox{$\GibbsRiskemp{S} \!\defeq \Esp_{\h\sim \post} \Hriskemp{S}$}.
Classical PAC-Bayesian theory provides high-probability upper bounds on the generalization gap for the Gibbs risk \citep[\eg,][]{mcallester2003pac,catoni2007pac,Seeger02,Maurer2004}.
 We recall below the generalization bound of \citet{Seeger02}  \citep[improved by][]{Maurer2004}, where the generalization 
 gap is captured by the KL-divergence between two Bernoulli distributions:
 $\kl(q\| p) \defeq q \ln(\frac{q}{p}){+}(1{-}q) \ln(\frac{1{-}q}{1{-}p})$.
 \allowdisplaybreaks[4]
\begin{theorem}[\citealp{Seeger02, Maurer2004}]\label{theorem:classical_kl}
    For any distribution $\Pcal$, any hypothesis set $\Hcal$, any prior $\prior$ on~$\Hcal$, and $\delta\!\in\! (0,1]$,  with probability at least $1{-}\delta$ on the random choice $S\sim \Pcal^m$, we have for any distribution $\post$ on $\Hcal$,
\begin{align*}
&\kl\left(\GibbsRiskemp{S}\middle\|\GibbsRisk{\Pcal}\right)\leq \frac{1}{m}\!\left[\KL(\post\|\prior){+}\ln\frac{2\sqrt{m}}{\delta}\right],
\end{align*}
where $\KL(\post\|\prior)\!\defeq\! \Esp_{\h \sim \post} \mathrm{ln}\frac{\post(\h)}{\prior(\h)}$ is the KL-divergence. 
\end{theorem}
\Cref{theorem:classical_kl} gives a lower and an upper bound (\textit{resp.} denoted by $^\downarrow$ and $^\uparrow$) 
on the Gibbs risk, \ie, under the same assumptions, we simultaneously have 
\begin{align}
\label{eq:borne-sup-kl}&\GibbsRisk{\Pcal} \le \underbrace{\klup\!\left(\GibbsRiskemp{S} \middle\| \frac{1}{m}\!\left[\KL(\post\|\prior){+}\ln\tfrac{2\sqrt{m}}{\delta}\right]\!\right)}_{\displaystyle
\bound_{S}^{\uparrow}(\post, \prior, \delta, \Gibbs_{\post})}, \\[1mm]
\label{eq:borne-inf-kl}\mbox{and }&\GibbsRisk{\Pcal} \ge \underbrace{\kldown\!\left(\GibbsRiskemp{S}\middle\|\frac{1}{m}\!\left[\KL(\post\|\prior){+}\ln \tfrac{2\sqrt{m}}{\delta}\right]\!\right)}_{\displaystyle \bound_{S}^{\downarrow}(\post, \prior, \delta, \Gibbs_{\post})},
\end{align}
with $\klup(q\|\epsilon) \defeq \max_{r\in[0,1]} \{r : \kl(q\|r)\leq \epsilon\}$ and $\kldown(q\|\epsilon) \defeq \min_{r\in[0,1]} \{r : \kl(q\|r)\leq \epsilon\}$. The above bounds hold for the stochastic Gibbs classifier~$\Gibbs_\post$.
While studying $\Gibbs_\post$ is of interest for randomized algorithms \citep[\eg,][]{DziugaiteR17}, in many practical settings one aims to learn a deterministic model. 
Some PAC-Bayesian approaches proposed to derive generalization bounds for deterministic classifiers.
One such approach relies on commonly named \emph{disintegrated} or \emph{derandomized} PAC-Bayesian bounds \citep[\eg,][]{catoni2007pac,blanchard2007,rivasplata2020,viallard2024leveraging,viallard2024general}.
They provide high-probability bounds over the random choice of $S$, but also over the random choice of a single model $h$ drawn from a learned~$\rho$.
As a result, although the bound is stated for a single $h$, this model remains random.

\looseness=-1
To obtain PAC-Bayesian bounds for a ``fully'' deterministic classifier, another approach consists in considering the weighted majority vote over $\Hcal$ according to $\posterior$, where each $h\!\in\!\Hcal$ is weighted according to its probability $\post(h)$, \ie,
\begin{align} \label{eq:mv}
& \forall x\in\Xcal,\  \MajVote{\rho}(x) \defeq \Esp_{\h\sim\post} \h(x)\,.
\end{align}
A family of work establishes inequalities that upper-bound the majority vote's risk $\MVrisk{\Pcal}{\post}$ in terms of the Gibbs risk $\GibbsRisk{\Pcal}$ \citep[\eg,][]{langford2002pac,mcallester2003pac,Lacasse06,MasegosaLorenzenIgelSeldin2020,RoyLavioletteMarchand2011,GermainLacasseLavioletteMarchandRoy2015,Leblanc25}.
These results imply that an upper bound on $\GibbsRisk{\Pcal}$ leads to an upper bound on $\MVrisk{\Pcal}{\post}$, up to additional terms depending on the inequality.
The simplest and most classical inequality \cite{langford2002pac} is 
\begin{align}
\label{eq:2gibbs}
\MVrisk{\Pcal}{\post}\leq2\GibbsRisk{\Pcal}.
\end{align} 
Consequently, any high-probability upper bound on the Gibbs risk (\eg, \Cref{eq:borne-sup-kl}) yields a high-probability upper bound on the majority vote's risk, up to a factor~$2$.
When one aims to learn a low-error classifier, such upper bounds are usually sufficient.
However, in this work we study group fairness measures that can be expressed as a difference of risk.
In this case, deriving a generalization bound requires both upper and lower bounds, since bounding a difference involves controlling deviations on both sides.

\section{Group Fairness Setting}
\label{sec:fairness}
\looseness=-1
PAC-Bayesian bounds are usually used to certify generalization capacities in terms of accuracy. 
As we discuss in \Cref{sec:related_works}, only a few results exist that provide such certification for fairness abilities. 
We address this gap by deriving PAC-Bayesian generalization bounds for fairness in supervised binary classification under the group fairness setting~\citep[see, \eg,][]{dwork2012fairness}, which we recall below.

\subsection{Setting and Fairness Measures}\label{section:fair_learning}
\looseness=-1
We now assume each input $x\!\in\! \Xcal$ to be associated with a binary sensitive attribute $g \!\in\! \AttrSpace\!=\!\{\GrOne,\GrTwo\}$, which indicates the subgroup to which $x$ belongs.
We consider an unknown distribution $\Dcal$ over $\Xcal{\times}\AttrSpace{\times}\Ycal$, so that an example is a triplet $(x,g,y)\!\sim\! \Dcal$.
Accordingly, a learning sample is now \mbox{$S\!=\!\{(x_i,g_i,y_i)\}_{i=1}^m \!\sim\! \Dcal^m$}.
Let denote $\Dcal_{|g}\!=\! \Dcal\big((x_i,y_i){\mid} g_i{=}g\big)$ the conditional distribution of $(x_i,y_i)$ given $g_i\!=\!g$.
Therefore, the true risk associated with a subgroup $g \! \in \!\AttrSpace$ is given by
\begin{align*}
    \Hrisk{\Dcal_{|g\!}} \defeq \Esp_{(x,y) \sim \Dcal_{|g}}\ell(h(x),y)\,.
\end{align*}
The empirical risk associated with a subgroup $g$ is computed using the part of the learning sample $S$ conditioned on the value $g$ of the sensitive attribute $S_{|g} \!=\! \{(x_i,g_i,y_i)\! \in\! S : g_i \!=\! g\}$, that is,
\begin{align*}
\GrHriskemp{S}{|g\!} \defeq \frac{1}{m_g} ~\sum_{\substack{i=1}}^{m_g} ~\ell(\h(x_i),y_i)\,,
\end{align*}
with $m_g$ the size of $\Sg$.
To ensure a fair decision between the subgroups for a classifier $h$, the objective is to limit the discrepancy between the risks of the subgroups,
\ie,  $\Hrisk{\DgOne\!}$ and $\Hrisk{\DgTwo}$ should be similar.

\textbf{General fairness risk.} We define this risk discrepancy as the absolute difference between the two subgroup risks, called \textit{true fairness risk},
\begin{align}\label{eq:fair_risk}
\FairRisk{\Dcal}{\h} \defeq \big| \Hrisk{\Dcal_{|\GrOne\!}} - \Hrisk{\Dcal_{|\GrTwo\!}} \big|\,,
\end{align}
and the \textit{empirical fairness risk} is 
\begin{align}\label{eq:fair_risk_emp}
\FairRiskemp{S}{\h} \defeq \big| \Hriskemp{S_{|\GrOne\!}} - \Hriskemp{S_{|\GrTwo\!}} \big|\,.
\end{align}
A classifier $h$ is then considered fair when $\FairRisk{\Dcal}{h}\!\approx\!0$.
Note that, $\FairRisk{\Dcal}{\Gibbs_\post}$ denotes the true fairness Gibbs risk and $\FairRisk{\Dcal}{\MajVote{\rho}}$ the true fairness risk of the majority vote.
As illustrated by the following examples, \Cref{eq:fair_risk} provides a general framework that encompasses fairness measures expressible as differences of risks.
This is made possible through appropriate choices of  loss~$\ell$ and  distribution $\Dcal$, and it applies identically to $h\!\in\!\Hcal$, or $G_{\rho}$, or $\MajVote{\rho}$.

\looseness=-1
\textbf{Demographic Parity} \citep[DP,][]{dwork2012fairness}\textbf{.}
A classifier is said to satisfy DP if its predictions are independent of the sensitive attribute.
In our setting, this notion can be expressed as a risk discrepancy by choosing the loss $\ell(h(x),y)\!=\!\ell(h(x),+1)$.
We have
\begin{align*}
    \DemParity{\Dcal}{\h} \defeq\ & \left| \Hrisk{\Dcal_{\Xcal|\GrOne\!}} - \Hrisk{\Dcal_{\Xcal|\GrTwo\!}} \right|,
\end{align*}
where $\Dcal_{\Xcal|g}$  is the marginal distribution of $\Xcal$ restricted to group $g$.
A classifier is considered fair under DP if it has the same expected positive prediction across sensitive groups, \ie,
$\Esp_{x \sim \Dcal_{\Xcal|\GrOne\!}} \ell(h(x),+1) \!=\! \Esp_{x \sim \Dcal_{\Xcal|\GrTwo\!}} \ell(h(x),+1)$.

\looseness=-1
\textbf{Equalized Odds} \citep[EO,][]{hardt2016equality}\textbf{.}
A classifier satisfies EO if its predictions are independent of the sensitive attribute conditionally on the true label. 
In fact, EO is related to DP in that it compares expected predictions across groups.
The key difference is that the risk discrepancy is performed separately for each class.
For a class $y\!\in\!\Ycal$, and loss $\ell(h(x),y)\!=\!\ell(h(x),+1)$, we have 
\begin{align*}
\nonumber \EqualOdds{\Dcal}{\h}{y}  \defeq\  & \left| \Hrisk{\Dcal_{\Xcal|y,\GrOne}} - \Hrisk{\Dcal_{\Xcal|y,\GrTwo}} \right|,
\end{align*}
with $\Dcal_{\Xcal|y,g}$ the marginal distribution of $\Xcal$ restricted to class $y$ and group $g$.
The overall EO is obtained by aggregating 
the average between the classes of $\Ycal\!=\!\{-1,+1\}$:
\begin{align*}
\EqualOdds{\Dcal}{\h}{}{\defeq}\!\!\!\underset{y\sim \Dcal_{\Ycal}}{\Pbb}\![y\!=\!1]\, \EqualOdds{\Dcal}{\h}{+1}+\!\! \underset{y\sim \Dcal_{\Ycal}}{\Pbb}\![y\!=\!-1]\,\EqualOdds{\Dcal}{\h}{-1}\!.
\end{align*}
Intuitively, this means that $h$ should behave similarly across sensitive groups on both classes. 

\looseness=-1
\textbf{Equal OPportunity} \citep[EOP,][]{hardt2016equality}\textbf{.}
A classifier is said to satisfy EOP if its predictions are independent of the sensitive attributes for one class of interest $y\! \in\! \Ycal$.
When this class is $y\!=\!+1$,  with $\ell(h(x),y)\!=\!\ell(h(x),+1)$, EOP is 
\begin{align*}
    \EqualOpp{\Dcal}{h} \defeq \EqualOdds{\Dcal}{\h}{+1}.
\end{align*}

\subsection{Generalization Bounds for Fairness}
\label{sec:related_works}

\looseness=-1
\textbf{Related Works.}
There exist only a few works on generalization bounds for fairness measures. 
Among them, some concern the derivation of bounds for specific fairness measures, such as \citet{woodworth2017} for EO, or \citet{agarwal2018} for DP and EO, or \citet{denis2024} for DP.
It is noteworthy that all these bounds are derived for specific algorithms and classifiers. 
Recently, \citet{laakom2025fairness} established an information-theoretic general framework to derive generalization bounds for various fairness measures that they applied to DP and EO.
While their results apply to a broad class of models, their resulting bounds do not lead to self-bounding algorithms. 
They provide post-hoc guarantees on the fairness of a learned predictor that are not directly exploitable as learning objectives.

\looseness=-1
\textbf{A PAC-Bayesian Bound.}
The most closely related work is the PAC-Bayesian EOP-bound for the stochastic Gibbs classifier proposed by \citet{oneto2019,oneto2020}.
For comparison purposes, we present in \Cref{theorem:oneto_general} a generalization of \citet[][Th.\,1]{oneto2020} to the general fairness risk of \Cref{eq:fair_risk}. 
We discuss this result further in Appendix~\ref{appendix:Oneto}.
\begin{theorem}[Generalization of Th.\,1 of \citealp{oneto2020}]\label{theorem:oneto_general}
  For any distribution $\Dcal$, any hypothesis set $\Hcal$, any prior $\Pcal$ on $\Hcal$, any $\delta\in (0,1]$, with probability at least $1-\delta$ on the random choice $S\sim \Dcal^m$, we have for every $\post$ on $\Hcal$,
\begin{align*}
\FairRisk{\Dcal}{\Gibbs_\post} 
\le \left| \GrGibbsRiskemp{S}{\GrTwo} - \GrGibbsRiskemp{S}{\GrOne}\right|  +  \sqrt{\frac{\displaystyle \KL(\post \| \prior) + \ln\tfrac{4\sqrt{m_{\GrOne}}}{\delta}}{2 m_{\GrOne}}}  +  \sqrt{\frac{\displaystyle\KL(\post \| \prior) + \ln\tfrac{4\sqrt{m_{\GrTwo}}}{\delta}}{2 m_{\GrTwo}}}\,.
\end{align*}
\end{theorem}
\begin{proof} 
We apply the \citeauthor{oneto2020}'s proof process on the general fairness risk (instead of EOP). 
We have 
\allowdisplaybreaks[4]
 \begin{align*}
          |\FairRisk{\Dcal}{\Gibbs_\post} - \FairRiskemp{S}{\Gibbs_\post} |  
         & = \left| | \GrGibbsRisk{\Dcal}{|\GrOne\!}- \GrGibbsRisk{\Dcal}{|\GrTwo\!}|  - | \GrGibbsRiskemp{S}{|\GrOne} - \GrGibbsRiskemp{S}{|\GrTwo}| \right|\\
        &\leq\   |\GrGibbsRisk{\Dcal}{|\GrOne\!} - \GrGibbsRiskemp{S}{|\GrOne}   | + | \GrGibbsRisk{\Dcal}{|\GrTwo\!} - \GrGibbsRiskemp{S}{|\GrTwo}  |.
    \end{align*}
We apply \citet{mcallester2003pac}'s PAC-Bayes bound (see \Cref{theorem:mcallester}) separately to each group, with probability at least $1{-}\frac{\delta}{2}$.
A union bound on the groups gives the result.
\end{proof}
Although the results of \citet{oneto2020,oneto2019} are purely theoretical and limited to EOP, they suggest a way (recalled in Appendix~\ref{appendix:Oneto}) to define $\prior$ and $\posterior$ to favor classifiers in~$\Hcal$ that achieve both good accuracy and good EOP, following a standard method in PAC-Bayes theory \citep[\eg,][]{catoni2007pac,lever2013}.
However, as noted by the authors, this method is hard to implement and may lead to numerical instabilities, especially when dealing with continuous distributions (thus, it is not evaluated in their work).
Note that, this method is not self-bounding.

That said, since $\FairRisk{\Dcal}{\Gibbs_\post}$ involves a deviation between the subgroup risks, \Cref{theorem:oneto_general} only holds for the stochastic Gibbs classifier and cannot be used to obtain upper bounds for the risk of the deterministic majority vote~$\FairRisk{\Dcal}{\MajVote{\posterior}}$.

\section{Fairness Generalization Bounds}\label{section:fairness_generalization_bounds}
\looseness=-1
This section focuses on the derivation of PAC-Bayesian generalization bounds for the general fairness risk $\FairRisk{\Dcal}{\cdot}$ (\Cref{eq:fair_risk}), paving the way to our \emph{fair self-bounding learning algorithm} (\Cref{sec:algo}).
 We first derive, in \Cref{sec:sto-fair}, a tighter bound than the one of \Cref{theorem:oneto_general} for the stochastic Gibbs classifier by relying on \Cref{theorem:classical_kl} (known to yield tighter bounds than \Cref{theorem:mcallester} on which \Cref{theorem:oneto_general} is based).
Then, by leveraging a recent result in PAC-Bayes that allows adapting guarantees from stochastic to deterministic classifiers \citep{Leblanc25}, we turn to deterministic classifiers 
in \Cref{sec:det-fair}, and to a computable specialization to the majority vote in \Cref{sec:mv-fair}.
Overall, our PAC-Bayesian bounds yield practical and theoretical benefits: \textit{(i)}~They can be instantiated to fairness measures that can be expressed as a risk discrepancy; \textit{(ii)}~They are computable from a learning sample $S$; \textit{(iii)}~When combined  with a PAC-Bayesian generalization bound on the predictive risk $\Risk_{\Dcal}(\cdot)$, they lead to a self-bounding algorithm to learn a model that directly minimizes a generalization bound on the trade-off between fairness and predictive risks.

\subsection{Stochastic Classifier}
\label{sec:sto-fair}

\looseness=-1
\Cref{theorem:gibbs_fairness_bound} states our PAC-Bayesian bound on the general fairness risk of the Gibbs classifier $\FairRisk{\Dcal}{G_\post}$.
\begin{theorem}\label{theorem:gibbs_fairness_bound}
    For any distribution $\Dcal$, any hypothesis set~$\Hcal$, any prior $\prior$ on $\Hcal$, any $\delta\in (0,1]$, with probability at least $1{-}\delta$ on the random choice $S{\sim}\Dcal^m$, we have for any $\post$ on~$\Hcal$,
    \begin{align*}
    \FairRisk{\Dcal}{\Gibbs_\post} \le \max &\Big\{\bound_{S_{|a\!}}^{\uparrow}(\post, \prior, \tfrac{\delta}{4}, \Gibbs_{\post}) -\bound_{S_{|b\!}}^{\downarrow}(\post, \prior, \tfrac{\delta}{4}, \Gibbs_{\post}),  \bound_{S_{|b\!}}^{\uparrow}(\post, \prior, \tfrac{\delta}{4}, \Gibbs_{\post}) - \bound_{S_{|a\!}}^{\downarrow}(\post, \prior, \tfrac{\delta}{4}, \Gibbs_{\post})\Big\}.
\end{align*}
\end{theorem}
\begin{proof} We apply \Cref{theorem:classical_kl} to bound the Gibbs risk on each group with probability $1{-}\frac{\delta}{2}$:
 \Cref{eq:borne-sup-kl} upper-bounds $\GrGibbsRisk{\Dcal}{|\GrOne\!}$, \Cref{eq:borne-inf-kl} lower-bounds $\GrGibbsRisk{\Dcal}{|\GrTwo\!}$.
We combine them by a union bound. 
With probability at least $1{-}\tfrac{\delta}{2}$ on $S\!\sim\!\Dcal^m$, we have, for any $\post$ on~$\Hcal$,
$$
\GrGibbsRisk{\Dcal}{|\GrOne\!}\!-\!\GrGibbsRisk{\Dcal}{|\GrTwo\!} \leq  \bound_{S_{|a\!}}^{\uparrow}(\post, \prior, \tfrac{\delta}{4}, \Gibbs_{\post}) \!-\! \bound_{S_{|b\!}}^{\downarrow}(\post, \prior, \tfrac{\delta}{4}, \Gibbs_{\post}).
$$
Since $\FairRisk{\Dcal}{G_\post}\!=\!|\GrGibbsRisk{\Dcal}{\GrOne\!}\!-\!\GrGibbsRisk{\Dcal}{|\GrTwo\!}|$, we apply the same argument symmetrically. 
By a union bound, the resulting bounds hold simultaneously.
Taking their maximum yields the desired result.
\end{proof}
In contrast to the proof of \Cref{theorem:oneto_general}, which is restricted to McAllester's bound (\Cref{theorem:mcallester}), a key feature of our proof is its generality: It can be instantiated with any PAC-Bayesian  bound for the Gibbs risk that provides both a lower and an upper bound.
Our result in \Cref{theorem:gibbs_fairness_bound} is based on the Seeger's bound (\Cref{theorem:classical_kl}), which is tighter than McAllester's bound due to Pinsker's inequality: 
$\forall (p,q)\!\in\![0,1]^2,\   2(q{-}p)^2\!\leq\! \kl(q\|p)$ \citep{wu2017lecture,canonne2022short}.
Thus, our bound is tighter than \Cref{theorem:oneto_general}.

\looseness=-1
While the stochastic Gibbs classifier is by nature the classical object of study of PAC-Bayes and shows interest for randomized algorithms, quantities such as $\FairRisk{\Dcal}{\Gibbs_\post}$ might not be representative of its actual fairness: The next example illustrates a situation where the DP of each classifier in $\Hcal$ is high, but where $\FairRisk{\Dcal}{\Gibbs_\post}$ fails to convey this information.
\begin{example}\label{ex:test}
    Let $\Hcal\! = \!\{h_1, h_2\}$ and \mbox{$\post(h_1), \post(h_2)\!=\! 0.5$}. 
    Suppose that \mbox{$\Risk_{\Dcal_{|\GrOne\!}}(h_1)\!=\!\Risk_{\Dcal_{|\GrTwo\!}}(h_2)=1$}, and \mbox{$\Risk_{\Dcal_{|\GrTwo\!}}(h_1)\!=\!\Risk_{\Dcal_{|\GrOne\!}}(h_2)\!=\!0$}.
    Since \mbox{$\Risk_{\Dcal_{|\cdot\!}}(\Gibbs_\post)\!=\!\frac{1}{2}(1\!+\!0)\!=\!\frac{1}{2}$}, we have
\begin{itemize}
    \item $\FairRisk{\Dcal}{h_1}=\left|\Risk_{\Dcal_{|\GrOne\!}}(h_1)-\Risk_{\Dcal_{|\GrTwo\!}}(h_1)\right| = |1-0| = 1$;
    \item $\FairRisk{\Dcal}{h_2}=\left|\Risk_{\Dcal_{|\GrOne\!}}(h_2)-\Risk_{\Dcal_{|\GrTwo\!}}(h_2)\right| = |0-1| = 1$;
    \item $\FairRisk{\Dcal}{\Gibbs_\post}\!=\!\left|\Risk_{\Dcal_{|\GrOne\!}}(\Gibbs_\post)-\Risk_{\Dcal_{|\GrTwo\!}}(\Gibbs_\post)\right| = \left|\tfrac{1}{2}{-}\tfrac{1}{2}\right| = 0$.
\end{itemize}
\end{example}
Therefore, in the following sections, we focus on the derivation of generalization bounds for deterministic classifiers.

\subsection{Deterministic Classifiers}
\label{sec:det-fair}
\looseness=-1
Recall that classical inequalities that upper-bound the majority vote's risk~$\MVrisk{\Dcal}{\post}$ in terms of the Gibbs risk~$\Risk_\Dcal(\Gibbs_\post)$ (\eg, \Cref{eq:2gibbs}) allow PAC-Bayesian upper bounds to be transferred from $\Gibbs_\post$ to $\MajVote{\post}$.
However, such inequalities do not extend to differences of Gibbs risks, as required for the general fairness measure (as in \Cref{theorem:oneto_general,theorem:gibbs_fairness_bound}).
To overcome this limitation, we leverage a recent result by \citet{Leblanc25}, which decomposes the risk of a single deterministic classifier $h\! \in \!\Hcal$ into the Gibbs risk and two conditional expectations terms. 
\begin{restatable}[Risk decomposition, \citealp{Leblanc25}]{proposition}{Riskdecomp}
\label{proposition:triple_decomposition}
Assume that $\TripleC \!-\! \TripleB \neq 0$. Then, for any $\h \in \Hcal$ and distribution $\post$ on $\Hcal$, the true risk of $\h$ satisfies
\begin{align}
\label{eq:decomposition-ben}
&\Hrisk{\Dcal} = \frac{\GibbsRisk{\Dcal} - \TripleB}{\TripleC - \TripleB},\\
\nonumber \text{with}\  \TripleB  \defeq \! &\underset{(x, y) \sim \Dcal}{\Esp}\left[\underset{\h^{\prime} \sim \post}{\Esp} \ell\left(\h^{\prime}(x), y\right) \;\middle|\; \ell(\h(x), y)=0 \right],\\
\nonumber \text{and}\ \TripleC \defeq \! &\underset{(x, y) \sim \Dcal}{\Esp}\left[\underset{\h^{\prime} \sim \post}{\Esp} \ell\left(\h^{\prime}(x), y\right) \;\middle|\; \ell(\h(x), y)=1\right].
\end{align}
\end{restatable}
\begin{proof}
  The proof is recalled in Appendix~\ref{app:triple_decomposition}.
\end{proof}
\Cref{eq:decomposition-ben} relates the risk of a classifier $\Hrisk{\Dcal}$ to the Gibbs risk $\Risk_{\Dcal}(\Gibbs_\post)$ and the terms $\TripleB$ and $\TripleC$ that measure how $\post$ conditionally behaves on correct and incorrect predictions of $\h$.
This suggests that controlling $\Hrisk{\Dcal}$ requires more than solely a bound on $\Risk_{\Dcal}(\Gibbs_\post)$. 
Thus, bounding $\Hrisk{\Dcal}$ reduces to obtaining simultaneous bounds on three quantities: $\Risk_{\Dcal}(\Gibbs_\post)$, $\TripleB$, and $\TripleC$.

To be able to extend PAC-Bayesian bounds beyond the stochastic setting for $\FairRisk{\Dcal}{h}\!=\!|\Hrisk{\Dcal_{|\GrOne\!}} \!-\! \Hrisk{\Dcal_{|\GrTwo\!}}|$, we need both the following upper and lower bounds on the risk $\Hrisk{\Dcal}$ of a deterministic classifier.
\allowdisplaybreaks[4]
\begin{restatable}[Bounds on $\Hrisk{\Dcal}$]{lemma}{DetBounds}
\label{lemma:det_upper_lower}
For any distribution $\Dcal$, any hypothesis set~$\Hcal$, any prior $\prior$ on~$\Hcal$,
let $\post$ be a distribution on $\Hcal$, and $\h\! \in \!\Hcal$, and $\delta\!\in\! (0,1]$.\\
\textbf{Upper bound} \citep{Leblanc25}\textbf{:}
With probability at least $1{-}\delta$ on \mbox{$S \!\sim\! \Dcal^m$}, if\\[1mm]
$$
\GibbsRisk{\Dcal} \le \bound_{S}^{\uparrow}(\post, \prior, \delta, \Gibbs_{\post}),\quad \TripleB \ge \DownTripleB,\quad\ \TripleC \ge \DownTripleC,
$$
with $\DownTripleC > \DownTripleB$, then we have
\begin{align}
\label{lemma:det_upper}
\!\!\!\!\!\!\Risk_{\Dcal}(\h) \!\le\!
 \frac{\bound_{S}^{\uparrow}(\post, \prior, \delta, \Gibbs_{\post}) \!-\! \DownTripleB}{\DownTripleC \!-\! \DownTripleB} \!\defeq\! \displaystyle\BoundMV_{S_{}}^{\uparrow}(\post, \prior, \delta, h).\!\!
\end{align}
\textbf{Lower bound:}
With probability at least $1{-}\delta$ on $S \!\sim\! \Dcal^m$, if\\[1mm]
$$
\GibbsRisk{\Dcal} \ge \bound_{S}^{\downarrow}(\post, \prior, \delta, \Gibbs_{\post}),\quad \TripleB \le \UpTripleB,\quad\ \TripleC \le \UpTripleC,
$$
with $\UpTripleC > \UpTripleB$,
then we have
\begin{align}
\label{lemma:det_lower}
\!\!\!\!\!\!\Risk_{\Dcal}(\h) \!\ge\!
\frac{\bound_{S}^{\downarrow}(\post, \prior, \delta, \Gibbs_{\post})\! - \! \UpTripleB}{\UpTripleC \!- \! \UpTripleB} \! \defeq \! \BoundMV_{S_{}}^{\downarrow}(\post, \prior, \delta, h).\!\!
\end{align}
\end{restatable}
\begin{proof}
    The proof is deferred to Appendix~\ref{app:proof:det_upper_lower}.
\end{proof}
At this step, the upper $^\uparrow$ and lower $^\downarrow$ bounds on $\TripleB$ and $\TripleC$ are assumed to exist.
In fact, \Cref{lemma:det_upper,lemma:det_lower} formalize that high-probability upper and lower bounds on the risk of a deterministic classifier $\h$ is obtainable from a PAC-Bayesian bound on the 
Gibbs risk and additional bounds on the conditional expectations $\TripleB$ and $\TripleC$.
 \Cref{sec:mv-fair} shows how one can compute these bounds when $h$ is a majority vote.

\looseness=-1
We now derive our main result: A high-probability upper bound for the fairness general measure of a deterministic $h$. 
\begin{theorem}\label{theorem:deterministic_abs_dif}
For any distribution $\Dcal$, any hypothesis set~$\Hcal$, any prior $\prior$ on $\Hcal$, and $\delta\!\in\! (0,1]$, with probability at least $1{-}\delta$ over the random choice of $S{\sim} \Dcal^m$ we have for any $\post$ over $\Hcal$ and any $\h \!\in\! \Hcal$
\begin{align*}
 \FairRisk{\Dcal}{\h}  \leq  \max\Big\{&\BoundMV_{S_{|\GrOne}\!}^{\uparrow}(\post, \prior, \tfrac{\delta}{4}, h) - \BoundMV_{S_{|\GrTwo}\!}^{\downarrow}(\post, \prior, \tfrac{\delta}{4}, h),~\BoundMV_{S_{|\GrTwo}\!}^{\uparrow}(\post, \prior, \tfrac{\delta}{4}, h) - \BoundMV_{S_{|\GrOne}\!}^{\downarrow}(\post, \prior, \tfrac{\delta}{4}, h)\Big\}.
\end{align*}
\end{theorem}
\begin{proof}
Using \Cref{proposition:triple_decomposition}, we rewrite the risk difference for a classifier.
For any $\Dcal$ on $\Xcal$, any $\Qcal$ on~$\Hcal$, and any $\h\!\in\!\Hcal$ \textit{s.t.} $\GrTripleB{\GrOne} \!\neq\! \GrTripleC{\GrOne}$ and $\GrTripleB{\GrTwo} \!\neq\! \GrTripleC{\GrTwo}$, we have
\begin{align*}  
\GrHrisk{\Dcal}{|\GrOne\!}\!{-}\GrHrisk{\Dcal}{|\GrTwo\!}\! 
=\! 
\frac{\GrGibbsRisk{\Dcal}{|\GrOne\!} \!{-} \GrTripleB{\GrOne\!\!}}{\GrTripleC{\GrOne\!\!} 
\!-\! 
\GrTripleB{\GrOne\!\!}}{-}\frac{\GrGibbsRisk{\Dcal}{|\GrTwo\!}\! {-} \GrTripleB{\GrTwo\!\!}}{\GrTripleC{\GrTwo\!\!} {-} \GrTripleB{\GrTwo\!\!}}
\end{align*}
We apply \Cref{lemma:det_upper} to upper-bound the risk decomposition for group $\GrOne$ with probability $1{-}\frac{\delta}{4}$, and \Cref{lemma:det_lower} to lower-bound the  decomposition for group $\GrTwo$ with probability $1{-}\frac{\delta}{4}$.
We combine them via a union bound to obtain, with probability at least $1{-}\tfrac{\delta}{2}$ on $S\!\sim\!\Dcal^m$, that for all $\post\mbox{ on } \Hcal$,\\[1mm]
\centerline{$\displaystyle
\GrHrisk{\Dcal}{|\GrOne}-\GrHrisk{\Dcal}{|\GrTwo}  \le \BoundMV_{S_{|\GrOne}\!}^{\uparrow}(\post, \prior, \tfrac{\delta}{4}, h)-\BoundMV_{S_{|\GrTwo}\!}^{\downarrow}(\post, \prior, \tfrac{\delta}{4}, h).$}\\[1mm]
Using \mbox{$|c{-}d|\!=\!\max\{c{-}d,\,d{-}c\},\,\forall c,d\in\Rbb$}, we apply the same argument symmetrically. 
By union bound, both results hold simultaneously.
Taking the maximum of the resulting upper bounds yields the desired result.
\end{proof}%
Importantly,~\Cref{theorem:deterministic_abs_dif} differs from standard results in PAC-Bayes in two ways: \textit{(i)}~It provides a generalization guarantee on the difference of two risks instead of an upper bound on a single risk, in the same fashion as~\Cref{theorem:gibbs_fairness_bound}; \textit{(ii)}~It relates to a single deterministic classifier $h$. 
Point \textit{(ii)} is of primary interest since, as we mentioned earlier, it is not possible to relate the difference between two risks defined for the deterministic majority vote $\MajVote{\post}$ with the difference of two Gibbs risks. 
Therefore, a bound on the latter quantity cannot be used to infer any information about $\MajVote{\post}$ (which is the major limitation of the result of~\citet{oneto2020}).

\subsection{Specialization to Weighted Majority Votes}
\label{sec:mv-fair}

\looseness=-1
To make \Cref{theorem:deterministic_abs_dif} computable, we need upper $^\uparrow$ and lower $^\downarrow$ bounds on the conditional expectations $\TripleB$ and~$\TripleC$.
To do so, we follow the approach of \citet{Leblanc25}. 
We assume a finite set of $n$ base classifiers $\Fcal\!=\!\{f_1,\dots,f_n\}$, with \mbox{$f_i\!:\!\Xcal\!\to\!\Ycal$}.
Each hypothesis is represented as a weighted majority vote on~$\Fcal$, parameterized by a weight vector $\wbf\! \in\! \mathbb{R}^n$.
Accordingly, the hypothesis set $\Hcal$ is defined as
\begin{align}\label{eq:hspace_mv}
\Hcal \defeq \Big\{ h_{\wbf} \defeq 
\sum_{i=1}^n w_i f_i \;\big|\; \wbf \in  \mathbb{R}^n \Big\}.
\end{align}
Thus, a hypothesis $h_{\wbf}\!\in\Hcal$ is fully specified by its associated weight vector $\wbf$.
In this setting, a posterior distribution $\post$ over $\Hcal$ is equivalently defined as a distribution over the weight vectors~$\wbf$.
Furthermore, we consider three families of posterior distributions, each parameterized by $\alphabf\!\in\!\Rbb^n$, which admit different sets of realizations $\wbf$:\\[1mm]
\phantom{a}\textbullet~Categorical $\cat(\alphabf)$, where $\wbf \!\in\! \{0,1\}^n$ with $\sum_i w_i \!=\!  1$,\\[1mm] 
\phantom{a}\textbullet~Dirichlet $\dir(\alphabf)$, where $\wbf$ lies on the probability simplex, \ie, $\wbf\! \in\! [0,1]^n$ with $\sum_i w_i \!=\! 1$, \\[1mm]
\phantom{a}\textbullet~Unit-variance Gaussian $\gaus(\alphabf, \Ibf_{n\times n})$, where  $\wbf\! \in\! \mathbb{R}^n$.

\looseness=-1
In this context, \citet{Leblanc25} proposed to learn the deterministic classifier $\halpha\!\in\!\Hcal$. 
When~$\posterior$ is Categorical or Gaussian \citep{langford2002pac}, $\alphabf$ acts both as the parameter of the studied classifier $\halpha$ and the parameters of a posterior distribution $\rho$ over $\Hcal$ parametrized by~$\alphabf$.\footnote{For the Dirichlet distribution, one needs to normalize $\alphabf$, see \citet[Sec.\,4.2,][]{Leblanc25} for details.}
The latter expresses a majority vote $\MajVote{\rho}$ as in \Cref{eq:mv}.
It remains to bound $\bD^\post(\h_\alphabf)$ and $\cD^\post(\h_\alphabf)$. 
The trick is to find their smallest and biggest attainable value. 
This can be done by applying the partition problem algorithm to $\alphabf$:
\begin{align*}
    \argmin_{\alphabf_1, \alphabf_2} \left\{\left| \sum_{\alpha \in \alphabf_1}\!\alpha-\!\!\!\sum_{\alpha \in \alphabf_2}\!\!\alpha\,\right|\!:\!\{\alphabf_1, \alphabf_2\}\, \text{is a partition of }\alphabf \right\}.
\end{align*}
\Cref{pr:bound_c_b} gives tractable upper and lower bounds on $\bD^\post(\h_\alphabf)$ and $\cD^\post(\h_\alphabf)$. 
The lower bounds are due to \citet{Leblanc25}, while the upper bounds are novel.

\begin{restatable}{proposition}{BoundsBC}\label{pr:bound_c_b}
    Let $\alphabf_1$ and $\alphabf_2$ be the result of the partition problem applied to~$\alphabf$. 
    Let \mbox{$\overset{\sim}{\alphabf} \!=\! \max\left(\sum_{\alpha\in\alphabf_1}\!\alpha, \sum_{\alpha\in\alphabf_2}\!\alpha\right)$}, and \mbox{$\overline{\alphabf} \!=\! \left|\sum_{\alpha\in\alphabf_1}\!\alpha\! -\! \sum_{\alpha\in\alphabf_2}\!\alpha\right|$}. 
    For any distribution~$\Dcal$, with \mbox{$\ell(h_{\alphabf}(x), y)\! =\! \mathbf{I}\big[\sign(h_{\alphabf}(x)) \neq y\big]$}, we have
\begin{enumerate}[itemsep=5mm,parsep=-3mm]
\item If  $\post=\cat(\alphabf)$, then \\[1mm]
       \centerline{$0 \leq \TripleBH \leq 1-\overset{\sim}{\alphabf}$,\ \ and \ \ $\overset{\sim}{\alphabf} \leq \TripleCH \leq 1$;}\\
\item If $\post=\dir(\alphabf)$, with $I_{0.5}(\cdot)$ the regularized incomplete beta function evaluated at $0.5$, then\\[1mm]
        \centerline{\phantom{and}\ \ $0 \leq \TripleBH \leq I_{0.5}\big(\overset{\sim}{\alphabf}, ~\|\alphabf\|_1-\overset{\sim}{\alphabf}\big)$,}\\[1mm]
        \centerline{and \ \ $I_{0.5}\big(\|\alphabf\|_1-\overset{\sim}{\alphabf}, ~\overset{\sim}{\alphabf}\big) \leq \TripleCH \leq 1$;}\\
\item If $\post=\gaus(\alphabf,\Ibf_{n\times n})$, with $\Phi(k)\!\defeq\!\frac12[1\!-\! \operatorname{Erf}(\frac{k}{\sqrt{2}})]$ and $\operatorname{Erf}(k)\!\defeq\!\frac{2}{\sqrt{\pi}}\int_{0}^{k} e^{-t^2}dt$, then
        \begin{align*}
         \Phi\left(\frac{\|\alphabf\|_1}{\sqrt{n}}\right) &\leq \TripleBH \leq \Phi\left(\frac{\overline{\alphabf}}{\sqrt{n}}\right),\\
        \text{and}\ \  1- \Phi\left(\frac{\overline{\alphabf}}{\sqrt{n}}\right) &\leq \TripleCH \leq 1-\Phi\left(\frac{\|\alphabf\|_1}{\sqrt{n}}\right).
        \end{align*}\\
\end{enumerate}
\end{restatable}

\begin{proof}
    The proof is deferred to Appendix~\ref{appendix:proof:bound_b_c}. 
\end{proof}
The conditions required to apply \Cref{lemma:det_upper_lower} (namely, $\DownTripleC \!>\! \DownTripleB$ and $\UpTripleC \!>\! \UpTripleB$) are satisfied for each considered distribution in \Cref{pr:bound_c_b}.
This allows us to apply 
\Cref{theorem:deterministic_abs_dif} to the deterministic majority vote~$\MajVote{\post}$, and thus to learn a majority vote by minimizing the resulting bound as described in the next section.

\section{A Fair Self-Bounding Algorithm}
\label{sec:algo}
\looseness=-1
We now derive our learning procedure, which consists of the minimization of a trade-off between the prediction risk and a fairness measure. 
For a given hypothesis set $\Hcal$ and a trade-off parameter $\lambda \!\in\! [0,1]$, we want to solve
\begin{align}
\label{eq:tradeoff_classic_h}
&\textstyle \argmin_{h\in\Hcal}~(1{-}\lambda)\Hrisk{\Dcal} + \lambda~\FairRisk{\Dcal}{h}\,.
\end{align}
Balancing this trade-off yields a fairness-aware learning problem \citep{menon2018tradeoff}, where the emphasis on fairness depends on 
$\lambda$.
Since $\Dcal$ is unknown, \Cref{eq:tradeoff_classic_h} cannot be solved directly. 
A common strategy is to use the empirical risk minimization (ERM)
principle by replacing $\Risk_\Dcal(\cdot)$ and $\RF_{\Dcal}(\cdot)$ with their empirical counterpart.
In such a situation, generalization guarantees are typically obtained post-hoc, by deriving bounds for the learned classifier.

\looseness=-1
A novelty of our approach is to directly optimize generalization guarantees on $\Risk_\Dcal(\cdot)$ and $\RF_{\Dcal}(\cdot)$. 
This amounts to replacing them in \Cref{eq:tradeoff_classic_h} by their respective generalization upper bounds using a union bound (the complete bounds are given in in~\Cref{theorem:trade-off_randomized,theorem:trade-off_deterministic} in Appendix~\ref{app:trade-off_proofs}).
In the PAC-Bayesian setting of \Cref{section:fairness_generalization_bounds}, where $\Hcal$ is defined as in \Cref{eq:hspace_mv} and the posterior~$\posterior$ is either $\cat(\alphabf)$, $\dir(\alphabf)$, or $\gaus(\alphabf)$, this leads to learning either a stochastic classifier\footnote{The Gibbs classifier on a set of majority vote is also referred to as the stochastic majority vote~\citep{zantedeschi2021learning}.} $\Gibbs_\post$ or a deterministic majority vote \mbox{$\MajVote{\post}\!=\!\halpha$}, depending on the bounds used as training objectives:
$\Gibbs_\post$ is learned by optimizing the bounds from \Cref{theorem:classical_kl} and \Cref{theorem:gibbs_fairness_bound}; $\MajVote{\post}$ is learned by optimizing the bounds from \Cref{lemma:det_upper_lower} (\Cref{lemma:det_upper}) and \Cref{theorem:deterministic_abs_dif}.
Given a learning sample~$S$, we respectively obtain the following learning objectives:
\begin{align}
\nonumber 
\objG_{S}(\posterior) =\ & (1{-}\lambda)\ \bound_{S}^{\uparrow}(\post, \prior, \tfrac{\delta}{2}, \Gibbs_{\post}) \\  & + \lambda\max\big\{\bound_{S_{|a\!}}^{\uparrow}(\post, \prior, \tfrac{\delta}{8}, \Gibbs_\post) \!-\!\bound_{S_{|b\!}}^{\downarrow}(\post, \prior, \tfrac{\delta}{8}, \Gibbs_{\post}),\ 
 \bound_{S_{|b\!}}^{\uparrow}(\post, \prior, \tfrac{\delta}{8}, \Gibbs_{\post}) \!-\! \bound_{S_{|a\!}}^{\downarrow}(\post, \prior, \tfrac{\delta}{8}, \Gibbs_{\post})\big\}\,,\label{eq:training_objective_stochastic}\\
\text{and}\quad \objMV_{S}(\posterior) =\ & (1{-}\lambda)\ \BoundMV_{S_{}}^{\uparrow}(\post, \prior, \tfrac{\delta}{2}, \MajVote{\post})\nonumber \\
& +\lambda\max\big\{\BoundMV_{S_{|\GrOne}\!}^{\uparrow}(\post, \prior, \tfrac{\delta}{8}, \MajVote{\post}) \!-\! \BoundMV_{S_{|\GrTwo}\!}^{\downarrow}(\post, \prior, \tfrac{\delta}{8}, \MajVote{\post}), \ \BoundMV_{S_{|\GrTwo}\!}^{\uparrow}(\post, \prior, \tfrac{\delta}{8}, \MajVote{\post}) \!-\! \BoundMV_{S_{|\GrOne}\!}^{\downarrow}(\post, \prior, \tfrac{\delta}{8}, \MajVote{\post})
\big\}.\label{eq:training_objective_det}
\end{align}
Here, the fairness measure, the confidence $\delta$, and the prior~$\prior$ are provided as inputs of the algorithm.
The associated minimization problems \mbox{$\argmin_{\posterior} \objG_{S}(\posterior)$} and \mbox{$\argmin_{\posterior} \objMV_{S}(\posterior)$} define self-bounding algorithms \citep{freund1998self}\footnote{Self-bounding algorithms have recently regained interest in PAC-Bayes, see \eg, \citet{rivasplata2022pac,viallard2023pac}.}, in the sense that the learned classifier is obtained by directly minimizing its own generalization bound.
Our overall minimization procedure is summarized in \Cref{alg:self_bounding}.
When learning a Gibbs classifier, we minimize $\objG_S(\post)$ using a standard gradient descent approach (Lines~1~to~9).
In contrast, when learning a majority vote $\MajVote{\posterior}$, directly minimizing $\objMV_{S}(\posterior)$ is not possible since the bounds on $\TripleB$ and $\TripleC$, as described in \Cref{pr:bound_c_b}, are by nature discrete.
To address this issue, we propose a two-phase optimization.
First, since the bounds in \Cref{theorem:classical_kl,theorem:gibbs_fairness_bound} (the content of $\objG_S(\post)$) are the backbone of the bounds in \Cref{lemma:det_upper_lower} and \Cref{theorem:deterministic_abs_dif} (the content of $\objMV_{S}(\posterior)$), we minimize $\objG_S(\post)$. 
Then, we apply the three following heuristics, ensuring that the bounds of \Cref{pr:bound_c_b} on $\bD^\post(\h_\alphabf)$ and $\cD^\post(\h_\alphabf)$, which play a crucial role in $\objMV_{S}(\posterior)$, are as tight as possible (Lines~10~to~12).

\begin{heuristic}\label{heur:1}
We clip the smallest absolute values of $\alphabf$ to~$0$, increasing $\overset{\sim}{\alphabf}$ and $\overline{\alphabf}$ in \cref{pr:bound_c_b}, until $\objMV_{S}(\posterior)$ does not decay anymore.
\end{heuristic}%
\begin{heuristic}\label{heur:2}
We apply a coordinate descent on the posterior values (increasing the largest components of $\alphabf$ and decreasing the smallest ones), directly increasing $\overset{\sim}{\alphabf}$ and $\overline{\alphabf}$ in \cref{pr:bound_c_b}, until $\objMV_{S}(\posterior)$ does not decay anymore.
\end{heuristic}%
\begin{heuristic}\label{heur:3}
Since $\overset{\sim}{\alphabf}$ and $\overline{\alphabf}$ grow linearly with $\|\alphabf\|_1$, we apply a multiplicative growing factor to each element in $\alphabf$ until $\objMV_{S}(\posterior)$ does not decay anymore.
\end{heuristic}

\begin{algorithm}[t] \small
\caption{Fair Self-Bounding Learning Algorithm}
\label{alg:self_bounding}
\begin{algorithmic}[1]

\REQUIRE Learning set $S\!=\!\{(x_i,y_i)\}_{i=1}^m$, epoch number~$T$, 
prior~$\prior$, confidence $\delta$, tradeoff param. $\lambda$, fairness measure (DP, EOP, or EP), Objective ($\objG_S(\post)$, or $\objMV_{S}(\posterior)$).
\FOR{$t = 1$ to $T$}
    \FORALL{mini-batches $U \subset S$ }
        \STATE Compute the empirical Gibbs risk $\GibbsRiskemp{U}$
        \STATE Compute the empirical fairness risk $\FairRiskemp{U}{\Gibbs_\post}$, 
        \STATE \qquad given the chosen fairness measure
        \STATE Compute the bound $\objG_U(\post)$
        \STATE Update $\post$ using a gradient step
    \ENDFOR
\ENDFOR
\IF{Objective is $\objMV_{S}(\posterior)$}
\STATE Update $\post$ using, in turn, \Cref{heur:1,heur:2,heur:3}
\ENDIF
\STATE \textbf{Output:} Posterior distribution $\post$

\end{algorithmic}
\end{algorithm}%

\begin{figure}[ht]
    \centering
    \begin{subfigure}{0.9\linewidth}
        \centering
        \includegraphics[width=\linewidth]{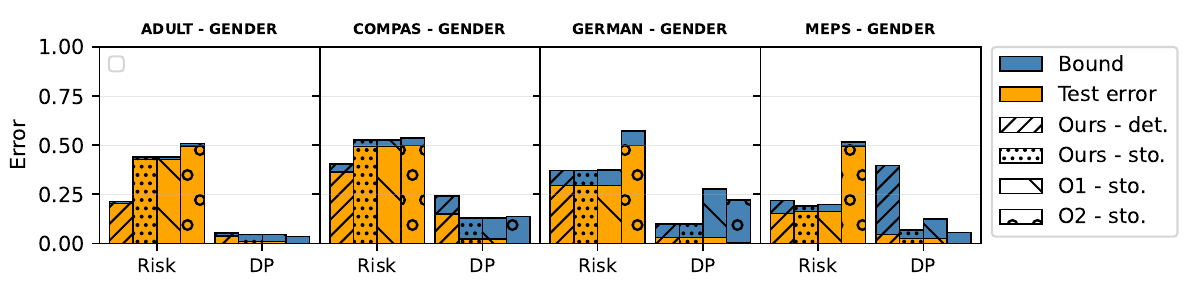}
        \caption{Sensitive attribute: Gender.}
        \label{fig:plot_DP_gender}
    \end{subfigure}

    \vspace{0.5em}

    \begin{subfigure}{0.9\linewidth}
        \centering
        \includegraphics[width=\linewidth]{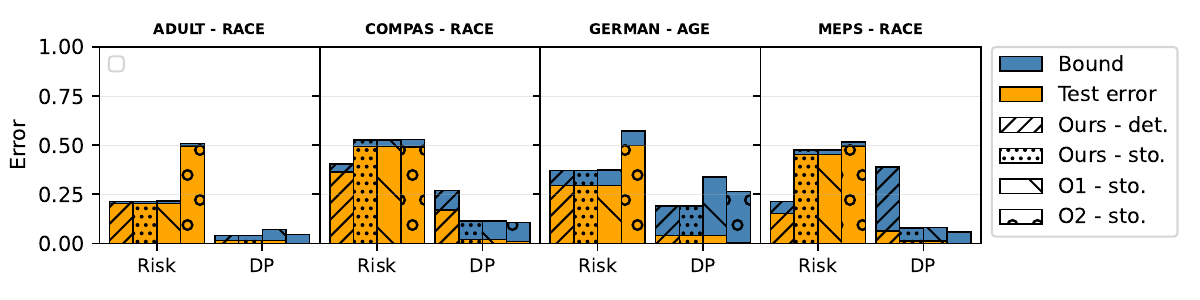}
        \caption{Sensitive attribute: Race or Age (depending on availability).}
        \label{fig:plot_DP_other}
    \end{subfigure}

    \caption{Test error and generalization bound of a stochastic majority vote classifier and its deterministic counterpart for Demographic Parity (DP).}
    \label{fig:plot_DP}
\end{figure}


\section{Experiments\protect\footnote{Our code is provided as supplementary material and will be publicly available upon publication.}}\label{sec:expe}

\looseness=-1
\textbf{Datasets.} Following the suggestions of \citet{HanCCW0Z024}, we use the AI Fairness 360 toolkit \citep{aif360} to empirically evaluate our theoretical framework on $5$ different datasets: Adult Census Income \citep[ADULT,][]{Kohavi96}, ProPublica Recidivism \citep[COMPAS,][]{compas}, German Credit \citep[GERMAN,][]{german_credit_data},
and Medical Expenditure Panel Surveys, version~1 \citep[MEPS,][]{AHRQ2015,AHRQ2016}. 
In each dataset, two sensitive attributes are considered, selected from $\{$age, gender, race$\}$. 
The test set consists of a random $20\%$ split of the available data. 
A thorough overview of the datasets is given in Appendix~\ref{app:datasets_overview}.

\looseness=-1
\textbf{Models.} We consider majority votes of data-independent hypotheses consisting of axis-aligned decision stumps, with thresholds evenly spread over the input space ($10$ per feature).
We independently experiment with all admissible posterior family (Categorical, Dirichlet, and Gaussian) and use a uniform prior from the same family as the posterior.

\textbf{Objective function.} Given a dataset and a sensitive attribute, the task is to minimize the learning objective given by Equations~\eqref{eq:training_objective_stochastic} or~\eqref{eq:training_objective_det}, instantiated by one fairness metric among DP, EO, or EOP, using~\Cref{alg:self_bounding}. 
Given a fixed objective function and a posterior distribution type, we consider $\lambda\!\in\!\{0, 0.1, 0.25, 0.5, 0.75, 0.9, 1\}$.

\textbf{Baselines.} We compare our approach to a self-bounding algorithm denoted by ``O1'',  obtained by the minimization of~\Cref{eq:training_objective_stochastic} when instantiated by the bound of~\citet{oneto2019} (\Cref{theorem:oneto_general}).
For completeness, we consider the method (which is not self-bounding) suggested by \citet{oneto2020} denoted by ``O2'' (recalled in Appendix~\ref{appendix:Oneto}) and for which we compute post-hoc bounds on $\FairRisk{\Dcal}{\Gibbs_\post}$ using a data-dependent prior. 
Developed for the EOP metric, we generalized it to the DP metric, but were not able to do so for the EO metric.
Moreover, O2, unstable for continuous distributions, is only used with the Categorical distribution. 

\textbf{Optimization.}
The models are trained using SGD with the Adam optimizer \cite{KingmaB14}.
We used a batch size equal to $1,024$, and a learning rate of $0.1$ with a scheduler reducing this parameter by a factor of $10$ with an epoch patience of $2$. 
The maximal number of epochs is $100$, and patience is $25$ for performing early stopping.

\textbf{Model selection.}
We report the results obtained by the model having the best average bound (over $5$ random seeds) for $\FairRisk{\Dcal}{h}$, out of the various $\lambda$ and distributions tested, for our method for deterministic classifiers and O1. 
For O2 and our method for a stochastic classifier, the criterion for model selection is the best bound on $\FairRisk{\Dcal}{\Gibbs_\post}$.

\looseness=-1
\textbf{Analysis.}
\Cref{fig:plot_DP} reports the results for Demographic Parity (numerical values, as well as results for EO and EOP, are provided in Appendix~\ref{app:experimentations_detailed}).

\textit{On bound tightness.}
Except for the deterministic classifier on MEPS, our proposed bounds are tight for both risk and fairness while achieving competitive or improved risk and DP values.
This behavior supports the use of self-bounding algorithms for learning models with their own guarantees.
Note that, since the deterministic classifier is learned from the stochastic one (Lines~1 to~9 of \Cref{alg:self_bounding}), the lack of tightness and degraded DP performance observed on MEPS (the highest-dimensional dataset) are likely due to the increased difficulty of applying the heuristics (Line~10) for this setting.
Additionally, when comparing O1 and our stochastic classifier, our bounds are either tighter (for ADULT, GERMAN, and MEPS-Gender) or similarly tight. 
This behavior is expected due to the nature of our bound based on a tighter PAC-Bayesian guarantee \citep[the one of][]{Seeger02}.
Lastly, since O2 is not a self-bounding algorithm, it does not provide tighter bounds than ours and is often worse than O1.\\[1mm]
\textit{On the learned trade-off risk/DP.}
Due to their similar nature, O1 and our stochastic classifier show very similar trade-offs. 
In contrast, our stochastic and deterministic classifiers may exhibit significantly different behaviors.
In particular, for ADULT-gender, COMPAS, and MEPS-race, the learned trade-off for the deterministic classifier favors lower risk at the cost of a higher DP (probably a consequence of enforcing determinism).
Note that O2 (stochastic and non-self-bounding) also behaves differently, yielding the lowest DP values on all datasets (except GERMAN-age); However, these improvements in DP are not significant and can come at the cost of a substantially higher risk (with the exception of ADULT-gender, COMPAS, and MEPS-race).
Interestingly, the deterministic classifier displays an opposite trend to O2. 
This may come from the fact that minimizing a bound captures additional information about the trade-off.

\textit{Take-home message.} 
Our results highlight that learning with self-certified fairness is not only theoretically sound but also practically effective: 
Tight PAC-Bayes bounds can be optimized, yielding \mbox{reliable risk/fairness trade-offs in practice}.

\section{Conclusion}
\label{sec:conclu}
\looseness=-1
In this paper, we introduce PAC-Bayesian generalization bounds on a general fairness risk and its extension to various fairness metrics for group fairness.
We instantiate such bounds on stochastic and deterministic classifiers defined as
majority votes and demonstrate the effectiveness of minimizing bounds to obtain fair majority votes with tight theoretical guarantees.
As future work, we plan to extend our framework to the multi-class setting and to non-binary sensitive attributes.

\subsubsection*{Acknowledgements.}
This work was supported in part by the French Project FAMOUS ANR-23-CE23-0019. Pascal Germain is supported by the NSERC Discovery grant RGPIN-2020-07223. Benjamin Leblanc is supported by a Mitacs Acceleration grant, in partnership with Intact Financial Corporation.  Paul Viallard is partially funded through Inria with the associate team PACTOL and the exploratory action HYPE.

\bibliography{references}

\begin{thebibliography}{48}
\providecommand{\natexlab}[1]{#1}
\providecommand{\url}[1]{\texttt{#1}}
\expandafter\ifx\csname urlstyle\endcsname\relax
  \providecommand{\doi}[1]{doi: #1}\else
  \providecommand{\doi}{doi: \begingroup \urlstyle{rm}\Url}\fi

\bibitem[Agarwal et~al.(2018)Agarwal, Beygelzimer, Dudik, Langford, and
  Wallach]{agarwal2018}
Agarwal, A., Beygelzimer, A., Dudik, M., Langford, J., and Wallach, H.
\newblock A reductions approach to fair classification.
\newblock In \emph{International Conference on Machine Learning}, 2018.

\bibitem[AHRQ(2015)]{AHRQ2015}
AHRQ.
\newblock Medical expenditure panel survey data: 2015 full year consolidated
  data file, 2015.

\bibitem[AHRQ(2016)]{AHRQ2016}
AHRQ.
\newblock Medical expenditure panel survey data: 2016 full year consolidated
  data file, 2016.

\bibitem[Angwin et~al.(2016)Angwin, Larson, Mattu, and Kirchner]{compas}
Angwin, J., Larson, J., Mattu, S., and Kirchner, L.
\newblock Machine bias: There's software used across the country to predict
  future criminals. and its biased against blacks.
\newblock ProPublica, 2016.

\bibitem[Bellamy et~al.(2019)Bellamy, Dey, Hind, Hoffman, Houde, Kannan, Lohia,
  Martino, Mehta, Mojsilovic, Nagar, Ramamurthy, Richards, Saha, Sattigeri,
  Singh, Varshney, and Zhang]{aif360}
Bellamy, R. K.~E., Dey, K., Hind, M., Hoffman, S.~C., Houde, S., Kannan, K.,
  Lohia, P., Martino, J., Mehta, S., Mojsilovic, A., Nagar, S., Ramamurthy,
  K.~N., Richards, J., Saha, D., Sattigeri, P., Singh, M., Varshney, K.~R., and
  Zhang, Y.
\newblock {AI Fairness} 360: An extensible toolkit for detecting,
  understanding, and mitigating unwanted algorithmic bias.
\newblock \emph{IBM Journal of Research and Development}, 2019.

\bibitem[Blanchard \& Fleuret(2007)Blanchard and Fleuret]{blanchard2007}
Blanchard, G. and Fleuret, F.
\newblock Occam's hammer.
\newblock In \emph{Conference on Learning Theory}, 2007.

\bibitem[Canonne(2022)]{canonne2022short}
Canonne, C.~L.
\newblock A short note on an inequality between kl and tv.
\newblock \emph{arXiv preprint arXiv:2202.07198}, 2022.

\bibitem[Caton \& Haas(2024)Caton and Haas]{caton2024surveyfairness}
Caton, S. and Haas, C.
\newblock Fairness in machine learning: {A} survey.
\newblock \emph{{ACM} Computing Survey}, 2024.

\bibitem[Catoni(2007)]{catoni2007pac}
Catoni, O.
\newblock {PAC-B}ayesian supervised classification: The thermodynamics of
  statistical learning.
\newblock \emph{IMS Lecture Notes Monograph Series}, 2007.

\bibitem[Czarnowska et~al.(2021)Czarnowska, Vyas, and
  Shah]{czarnowska2021quantifying}
Czarnowska, P., Vyas, Y., and Shah, K.
\newblock Quantifying social biases in nlp: A generalization and empirical
  comparison of extrinsic fairness metrics.
\newblock \emph{Transactions of the Association for Computational Linguistics},
  2021.

\bibitem[Denis et~al.(2024)Denis, Elie, Hebiri, and Hu]{denis2024}
Denis, C., Elie, R., Hebiri, M., and Hu, F.
\newblock Fairness guarantees in multi-class classification with demographic
  parity.
\newblock \emph{Journal of Machine Learning Research}, 2024.

\bibitem[Dwork et~al.(2012)Dwork, Hardt, Pitassi, Reingold, and
  Zemel]{dwork2012fairness}
Dwork, C., Hardt, M., Pitassi, T., Reingold, O., and Zemel, R.
\newblock Fairness through awareness.
\newblock In \emph{Innovations in Theoretical Computer Science Conference},
  2012.

\bibitem[Dziugaite \& Roy(2017)Dziugaite and Roy]{DziugaiteR17}
Dziugaite, G.~K. and Roy, D.~M.
\newblock Computing nonvacuous generalization bounds for deep (stochastic)
  neural networks with many more parameters than training data.
\newblock In \emph{Conference on Uncertainty in Artificial Intelligence}, 2017.

\bibitem[Freund(1998)]{freund1998self}
Freund, Y.
\newblock Self bounding learning algorithms.
\newblock In \emph{Conference on Computational Learning Theory}, 1998.

\bibitem[Germain et~al.(2015)Germain, Lacasse, Laviolette, Marchand, and
  Roy]{GermainLacasseLavioletteMarchandRoy2015}
Germain, P., Lacasse, A., Laviolette, F., Marchand, M., and Roy, J.
\newblock Risk bounds for the majority vote: From a {PAC-B}ayesian analysis to
  a learning algorithm.
\newblock \emph{Journal of Machine Learning Research}, 2015.

\bibitem[Han et~al.(2024)Han, Chi, Chen, Wang, Zhao, Zou, and Hu]{HanCCW0Z024}
Han, X., Chi, J., Chen, Y., Wang, Q., Zhao, H., Zou, N., and Hu, X.
\newblock {FFB:} {A} fair fairness benchmark for in-processing group fairness
  methods.
\newblock In \emph{International Conference on Learning Representations}, 2024.

\bibitem[Hardt et~al.(2016)Hardt, Price, and Srebro]{hardt2016equality}
Hardt, M., Price, E., and Srebro, N.
\newblock Equality of opportunity in supervised learning.
\newblock \emph{Advances in Neural Information Processing Systems}, 2016.

\bibitem[Hofmann(1994)]{german_credit_data}
Hofmann, H.
\newblock {Statlog (German Credit Data)}.
\newblock UCI Machine Learning Repository, 1994.

\bibitem[Kingma \& Ba(2015)Kingma and Ba]{KingmaB14}
Kingma, D.~P. and Ba, J.
\newblock Adam: {A} method for stochastic optimization.
\newblock In \emph{International Conference on Learning Representations}, 2015.

\bibitem[Kohavi(1996)]{Kohavi96}
Kohavi, R.
\newblock Scaling up the accuracy of naive-bayes classifiers: {A} decision-tree
  hybrid.
\newblock In \emph{International Conference on Knowledge Discovery and Data
  Mining}, 1996.

\bibitem[Laakom et~al.(2025)Laakom, Chen, Schmidhuber, and
  Bu]{laakom2025fairness}
Laakom, F., Chen, H., Schmidhuber, J., and Bu, Y.
\newblock Fairness overfitting in machine learning: An information-theoretic
  perspective.
\newblock In \emph{International Conference on Machine Learning}, 2025.

\bibitem[Lacasse et~al.(2006)Lacasse, Laviolette, Marchand, Germain, and
  Usunier]{Lacasse06}
Lacasse, A., Laviolette, F., Marchand, M., Germain, P., and Usunier, N.
\newblock {{PAC-B}ayes} bounds for the risk of the majority vote and the
  variance of the {G}ibbs classifier.
\newblock In \emph{Advances in Neural Information Processing Systems}, 2006.

\bibitem[Langford \& Shawe-Taylor(2002)Langford and
  Shawe-Taylor]{langford2002pac}
Langford, J. and Shawe-Taylor, J.
\newblock {PAC-B}ayes \& margins.
\newblock \emph{Advances in Neural Information Processing Systems}, 2002.

\bibitem[Leblanc \& Germain(2025)Leblanc and Germain]{Leblanc25}
Leblanc, B. and Germain, P.
\newblock A framework for bounding deterministic risk with {PAC-B}ayes:
  Applications to majority votes.
\newblock \emph{arXiv preprint arXiv:2510.25569}, 2025.

\bibitem[Lever et~al.(2013)Lever, Laviolette, and Shawe-Taylor]{lever2013}
Lever, G., Laviolette, F., and Shawe-Taylor, J.
\newblock Tighter {PAC-B}ayes bounds through distribution-dependent priors.
\newblock \emph{Theoretical Computer Science}, 2013.

\bibitem[Masegosa et~al.(2020)Masegosa, Lorenzen, Igel, and
  Seldin]{MasegosaLorenzenIgelSeldin2020}
Masegosa, A., Lorenzen, S.~S., Igel, C., and Seldin, Y.
\newblock Second order {PAC-B}ayesian bounds for the weighted majority vote.
\newblock In \emph{Advances in Neural Information Processing Systems}, 2020.

\bibitem[Maurer(2004)]{Maurer2004}
Maurer, A.
\newblock A note on the {PAC} bayesian theorem.
\newblock \emph{arXiv preprint arXiV:cs.LG/0411099}, 2004.

\bibitem[McAllester(1998)]{McAllester1998SomePT}
McAllester, D.~A.
\newblock Some {PAC-B}ayesian theorems.
\newblock \emph{Machine Learning}, 1998.

\bibitem[McAllester(2003)]{mcallester2003pac}
McAllester, D.~A.
\newblock {PAC-B}ayesian stochastic model selection.
\newblock \emph{Machine Learning}, 2003.

\bibitem[Mehrabi et~al.(2021)Mehrabi, Morstatter, Saxena, Lerman, and
  Galstyan]{mehrabi2021survey}
Mehrabi, N., Morstatter, F., Saxena, N., Lerman, K., and Galstyan, A.
\newblock A survey on bias and fairness in machine learning.
\newblock \emph{ACM computing surveys}, 2021.

\bibitem[Menon \& Williamson(2018)Menon and Williamson]{menon2018tradeoff}
Menon, A.~K. and Williamson, R.~C.
\newblock The cost of fairness in binary classification.
\newblock In \emph{Conference on Fairness, Accountability and Transparency},
  2018.

\bibitem[Mertens(2005)]{partition}
Mertens, S.
\newblock The easiest hard problem: Number partitioning.
\newblock In \emph{Computational Complexity and Statistical Physics}. Oxford
  University Press, 2005.

\bibitem[Oneto et~al.(2019)Oneto, Donini, and Pontil]{oneto2019}
Oneto, L., Donini, M., and Pontil, M.
\newblock {PAC-B}ayes and fairness: Risk and fairness bounds on distribution
  dependent fair priors.
\newblock In \emph{European Symposium on Artificial Neural Networks}, 2019.

\bibitem[Oneto et~al.(2020)Oneto, Donini, Pontil, and Shawe-Taylor]{oneto2020}
Oneto, L., Donini, M., Pontil, M., and Shawe-Taylor, J.
\newblock Randomized learning and generalization of fair and private
  classifiers: From {PAC-B}ayes to stability and differential privacy.
\newblock \emph{Neurocomputing}, 2020.

\bibitem[Rivasplata(2022)]{rivasplata2022pac}
Rivasplata, O.
\newblock \emph{{PAC-B}ayesian computation}.
\newblock PhD thesis, University College London, UK, 2022.

\bibitem[Rivasplata et~al.(2020)Rivasplata, Kuzborskij, Szepesvari, and
  Shawe-Taylor]{rivasplata2020}
Rivasplata, O., Kuzborskij, I., Szepesvari, C., and Shawe-Taylor, J.
\newblock {PAC-B}ayes analysis beyond the usual bounds.
\newblock In \emph{Advances in Neural Information Processing Systems}, 2020.

\bibitem[Roy et~al.(2011)Roy, Laviolette, and
  Marchand]{RoyLavioletteMarchand2011}
Roy, J., Laviolette, F., and Marchand, M.
\newblock From {PAC-B}ayes bounds to quadratic programs for majority votes.
\newblock In \emph{International Conference on Machine Learning}, 2011.

\bibitem[Seeger(2002)]{Seeger02}
Seeger, M.~W.
\newblock {PAC-B}ayesian generalisation error bounds for gaussian process
  classification.
\newblock \emph{Journal of Machine Learning Reasearch}, 2002.

\bibitem[Shawe-Taylor \& Williamson(1997)Shawe-Taylor and
  Williamson]{ShaweTaylor1997}
Shawe-Taylor, J. and Williamson, R.~C.
\newblock A {PAC} analysis of a bayesian estimator.
\newblock In \emph{Conference on Computational Learning Theory}, 1997.

\bibitem[Ustun et~al.(2019)Ustun, Liu, and Parkes]{ustun19a}
Ustun, B., Liu, Y., and Parkes, D.
\newblock Fairness without harm: Decoupled classifiers with preference
  guarantees.
\newblock In \emph{International Conference on Machine Learning}, 2019.

\bibitem[Valiant(1984)]{Valiant1984}
Valiant, L.
\newblock A theory of the learnable.
\newblock \emph{Communications of the ACM}, 1984.

\bibitem[Vapnik(1999)]{vapnik1999nature}
Vapnik, V.
\newblock \emph{The nature of statistical learning theory}.
\newblock Springer science \& business media, 1999.

\bibitem[Viallard(2023)]{viallard2023pac}
Viallard, P.
\newblock \emph{{PAC-B}ayesian Bounds and Beyond: Self-Bounding Algorithms and
  New Perspectives on Generalization in Machine Learning}.
\newblock PhD thesis, University Jean Monnet, France, 2023.

\bibitem[Viallard et~al.(2024{\natexlab{a}})Viallard, Emonet, Habrard, Morvant,
  and Zantedeschi]{viallard2024leveraging}
Viallard, P., Emonet, R., Habrard, A., Morvant, E., and Zantedeschi, V.
\newblock Leveraging {PAC-B}ayes theory and gibbs distributions for
  generalization bounds with complexity measures.
\newblock In \emph{International Conference on Artificial Intelligence and
  Statistics}, 2024{\natexlab{a}}.

\bibitem[Viallard et~al.(2024{\natexlab{b}})Viallard, Germain, Habrard, and
  Morvant]{viallard2024general}
Viallard, P., Germain, P., Habrard, A., and Morvant, E.
\newblock A general framework for the practical disintegration of
  {PAC-B}ayesian bounds.
\newblock \emph{Machine Learning}, 2024{\natexlab{b}}.

\bibitem[Woodworth et~al.(2017)Woodworth, Gunasekar, Ohannessian, and
  Srebro]{woodworth2017}
Woodworth, B., Gunasekar, S., Ohannessian, M.~I., and Srebro, N.
\newblock Learning non-discriminatory predictors.
\newblock In \emph{Conference on Learning Theory}, 2017.

\bibitem[Wu(2017)]{wu2017lecture}
Wu, Y.
\newblock Lecture notes on information-theoretic methods for high-dimensional
  statistics.
\newblock \emph{Lecture Notes for ECE598YW (UIUC)}, 2017.

\bibitem[Zantedeschi et~al.(2021)Zantedeschi, Viallard, Morvant, Emonet,
  Habrard, Germain, and Guedj]{zantedeschi2021learning}
Zantedeschi, V., Viallard, P., Morvant, E., Emonet, R., Habrard, A., Germain,
  P., and Guedj, B.
\newblock Learning stochastic majority votes by minimizing a {PAC-B}ayes
  generalization bound.
\newblock \emph{Advances in Neural Information Processing Systems}, 2021.

\end{thebibliography}
\bibliographystyle{bibliostyle}

\appendix

\section{Detailed discussion of the results of~\citet{oneto2020}}
\label{appendix:Oneto}

\citet{oneto2020} proposed a PAC-Bayesian fairness bound on the measure of Equal Opportunity~\cite{hardt2016equality} (we build upon the formalization of the Equal Opportunity risk from~\Cref{section:fair_learning}). 
Let $\Dcal_{\Xcal|+1,g}$ the marginal distribution of $\Xcal$ restricted to the class $y = +1$ and sensitive group $g \in \AttrSpace$. 
The corresponding true Gibbs risk for any group $g \in \AttrSpace$ is,
\begin{align*}
    \GibbsRisk{\Dcal_{\Xcal|+1,g}} =  \underset{h \sim \post}{\Esp} \underset{x \sim \Dcal_{\Xcal|+1,g}}{\Esp}\ell(h(x), +1).
\end{align*}

Similarly, for the empirical version, take the sample $S_{\Xcal|+1,g} = \{(x_i,g_i) \in S : g_i = g, y_i=+1\}$ and compute the risk,
\begin{align*}
\GibbsRiskemp{S_{\Xcal|+1,g}}  = \frac{1}{m_{+1,g}} ~\sum_{\substack{i=1}}^{m_{+1, g}}~  \Esp_{\h \sim \post} ~\h(x_i)
\end{align*}
with $m_{+1,g} =\left|S_{\Xcal|+1,g}\right|$. 
Therefore, in the binary fairness setting, the Equal Opportunity true risk is, 
\begin{align*}
\EqualOpp{\Dcal}{G_\post}
= \bigg| \GibbsRisk{\Dcal_{\Xcal|+1,\GrOne}} - \GibbsRisk{\Dcal_{\Xcal|+1,\GrTwo}} \bigg|
\end{align*}
and the empirical Equal Opportunity risk is, 
\begin{align*}
\EqualOppEmp{S}{G_\post}
= \bigg| \GibbsRiskemp{S_{\Xcal|+1,\GrOne}} - \GibbsRiskemp{S_{\Xcal|+1,\GrTwo}} \bigg|
\end{align*}

The upper bound on the EOP risk (\Cref{theorem:Oneto_EOP}) is obtained by instantiating the classical PAC-Bayesian bound of \citet{mcallester2003pac} (\Cref{theorem:mcallester}) to the group-conditional Gibbs risks involved in the Equal Opportunity criterion.

\begin{theorem}[\citealp{oneto2020}]\label{theorem:Oneto_EOP}
    For any distribution $\Dcal$, hypothesis set $\Hcal$, prior distribution $\Pcal$ over $\Hcal$, $\delta\in (0,1]$, we have with probability at least $1-\delta$ over the random choice $S\sim \Dcal^m$ that for every $\post$ over $\Hcal$:
\begin{align*}
 \bigg|&\EqualOppEmp{S}{\post} - \EqualOpp{\Dcal}{G_\post} \bigg| \le\sqrt{\frac{1}{2 m_{+11,\GrOne}} \left(\KL(G_\post \| \prior) + \ln\left(\frac{4\sqrt{m_{+1,\GrOne}}}{\delta}\right) \right)} + \sqrt{\frac{1}{2 m_{+1, \GrTwo}} \left(\KL(\post \| \prior) + \ln\left(\frac{4\sqrt{m_{+11,\GrTwo}}}{\delta}\right) \right)}.
\end{align*}
\end{theorem}
\begin{proof}
    Note that from the triangle inequality, we have
    \begin{align*}
        \bigg|\EqualOppEmp{S}{G_\post} - \EqualOpp{\Dcal}{G_\post} \bigg| &= \bigg| \big| \GibbsRiskemp{S_{\Xcal|+1,\GrOne}} - \GibbsRiskemp{S_{\Xcal|+1,\GrTwo}} \big| -  
        \big| \GibbsRisk{\Dcal_{\Xcal|+1,\GrOne}} - \GibbsRisk{\Dcal_{\Xcal|+1,\GrTwo}} \big|\bigg|\\
        &\leq \bigg| \GibbsRisk{\Dcal_{\Xcal|+1,\GrOne}} - \GibbsRiskemp{S_{\Xcal|+1,\GrOne}}  \bigg| + \bigg| \GibbsRisk{\Dcal_{\Xcal|+1,\GrTwo}} -\GibbsRiskemp{S_{\Xcal|+1,\GrTwo}} \bigg|.
    \end{align*}
    Then apply the classical PAC-Bayes bound (\Cref{theorem:mcallester}) separately to each subgroup, with probability at least $1{-}\frac{\delta}{2}$.
A union bound over the two subgroups gives the desired result.
\end{proof}

\begin{theorem}[\citealp{mcallester2003pac}]\label{theorem:mcallester}
    For any distribution $\Dcal$, hypothesis set $\Hcal$, prior distribution $\prior$ over $\Hcal$, $\delta\in (0,1]$,  with probability at least $1-\delta$ over the random choice $S\sim \Dcal^m$, we have for every $\post$ over $\Hcal$,
    $$\bigg| \GibbsRisk{\Dcal} - \GibbsRiskemp{S} \bigg|\leq \sqrt{\frac{1}{2m}\left[\KL(\post\|\prior)+\ln\frac{2\sqrt{m}}{\delta}\right]}.$$
\end{theorem}

To define the posterior $\post$ and the prior $\prior$, \citet{oneto2020}  propose to follow a standard PAC-Bayesian approach \citep[\eg,][]{catoni2007pac,lever2013} that intends to favor the classifier with good accuracy and EOP. Let
$\lambda \in [0,+\infty[$ which denotes the trade-off parameter between accuracy and fairness. One can define the posterior distribution under this approach, denoted by $\post_\Dcal$, and defined as
\begin{align}\label{eq:Q_oneto}
\post_\Dcal(h) \propto \exp\left(- \gamma\left[\Hriskemp{S} + \lambda\EqualOppEmp{S}{h}\right]\right),
\end{align}
Similarly, the prior distribution is defined as 
\begin{align}\label{eq:P_oneto}
\prior(h) \propto \exp\left(- \gamma\big[\Hrisk{\Dcal} + \lambda\EqualOpp{\Dcal}{h}\big]\right).
\end{align}

However, the prior distribution defined in Equation~\eqref{eq:P_oneto} cannot be computed in practice, as it explicitly depends on the true risk $\Hrisk{\Dcal}$ and the true Equal Opportunity risk $\EqualOpp{\Dcal}{h}$, which are both unknown. Using the definition of the prior $\prior$ and posterior $\post_\Dcal$ from Equations~\eqref{eq:P_oneto} and~\eqref{eq:Q_oneto}, \citet{oneto2020} derive the following bound on $\KL(\post_\Dcal\|\|\prior)$.

\begin{theorem}[\citealp{oneto2020}]\label{theorem:oneto_kl} Given the prior $\prior$ and posterior $\post_\Dcal$  defined in Equations~\eqref{eq:P_oneto} and~\eqref{eq:Q_oneto}, we have
\begin{align*}
\underset{S \sim \Dcal^m}{\Pbb}\left[\KL(\post_\Dcal\|\|\prior) \ge \KL_2(\delta, m, m_{+1,\GrOne}, m_{+1,\GrTwo})\right] \le 6\delta
\end{align*}
where
\begin{align*} 
& \KL_2\left(\delta, m, m_{+1,\GrOne}, m_{+1,\GrTwo}\right)=\varepsilon_1^2+2 \varepsilon_1 \sqrt{\varepsilon_2}+\varepsilon_2, \\ & \varepsilon_1= \gamma\left(\sqrt{\frac{1}{2 m}}+\lambda\left(\sqrt{\frac{1}{2 m_{+1,\GrOne}}}+\sqrt{\frac{1}{2 m_{+1,\GrTwo}}}\right)\right), \\ & \varepsilon_1= 2 \gamma\left(\sqrt{\frac{1}{2 m} \ln \left(\frac{2 \sqrt{m}}{\delta}\right)}+\lambda\left(\sqrt{\frac{1}{2 m_{+1,\GrOne}} \ln \left(\frac{2 \sqrt{m_{+1,\GrOne}}}{\delta}\right)} +\sqrt{\frac{1}{2 m_{+1,\GrTwo}} \ln \left(\frac{2 \sqrt{m_{+1,\GrTwo}}}{\delta}\right)}\right)\right) .
\end{align*}
    
\end{theorem}

Then given \Cref{theorem:oneto_kl},~\citet{oneto2020} obtain the following bounds on the risk and the EOP risk by plugging ~\Cref{theorem:oneto_kl} into~\Cref{theorem:Oneto_EOP}.

\begin{theorem}[\citealp{oneto2020}]\label{theorem:Oneto_final_data_dependent}
    Given the prior $\prior$ and posterior $\post_\Dcal$  defined in Equations~\eqref{eq:P_oneto} and~\eqref{eq:Q_oneto}, the risk and the fairness of the stochastic classifier can be bounded in the following way:
\begin{align*}
&\underset{S \sim \Dcal^m}{\Pbb}\left[\bigg| \Risk_{\Dcal}\!\left(\Gibbs_{\post_\Dcal}\right) - \Riskemp_{S}\!\left(\Gibbs_{\post_\Dcal}\right) \bigg|\leq \sqrt{\frac{1}{2m}\left[\KL_2\left(\delta, m, m_{+1,\GrOne}, m_{+1,\GrTwo}\right)+\ln\frac{2\sqrt{m}}{\delta}\right]}\right]\le 7\delta,\\
\text{and}\quad &\underset{S \sim \Dcal^m}{\Pbb}\Bigg[
\bigg|\EqualOppEmp{S}{\post_\Dcal} - \EqualOpp{\Dcal}{\post_\Dcal} \bigg|
\le
\sqrt{\frac{1}{2 m_{+1,\GrOne}}
\left(
\KL_2\!\left(\delta, m, m_{+1,\GrOne}, m_{+1,\GrTwo}\right)
+ \ln\!\left(\frac{2\sqrt{m_{+1,\GrOne}}}{\delta}\right)
\right)} \\
&\hspace{4.0cm} +
\sqrt{\frac{1}{2 m_{+1,\GrTwo}}
\left(
\KL_2\!\left(\delta, m, m_{+1,\GrOne}, m_{+1,\GrTwo}\right)
+ \ln\!\left(\frac{2\sqrt{m_{+1,\GrTwo}}}{\delta}\right)
\right)}
\Bigg] \le 8\delta.
\end{align*}
\end{theorem}

Even though the prior $\prior$ is defined using $\Hrisk{\Dcal}$ and  $\EqualOpp{\Dcal}{h}$, the result from~\Cref{theorem:Oneto_final_data_dependent} can be computed thanks to the fact that the $\KL_2\left(\delta, m, m_{+1,\GrOne}, m_{+1,\GrTwo}\right)w$ defined in~\Cref{theorem:oneto_kl} does not not require to access $\prior$. The posterior $\post_\Dcal$ defined~\Cref{eq:Q_oneto} gives the classifier $G_{\post_\Dcal}$ of the method denote O2 in~\Cref{sec:expe}, the reported bounds are given by~\Cref{theorem:Oneto_final_data_dependent}. 



\section{Proofs of the main results}

\subsection{Proof of \Cref{proposition:triple_decomposition}}\label{app:triple_decomposition}
\Riskdecomp*

\begin{proof}
    For any distribution $\Dcal$, distribution $\post$ over $\Hcal$ and classifier $h\in\Hcal$ such that $\TripleC\neq\TripleB$, we have:
\allowdisplaybreaks[4]
    \begin{align*}
        \GibbsRisk{\Dcal} = & \;  \Esp_{\h\sim \post} \Esp_{(x,y) \sim \Dcal}  \ell(\h(x),y), \\
        = & \; \underset{(x,y)\sim\Dcal}{\mathbb{P}} ( \ell(\h(x),y) = 1) \Esp_{(x,y) \sim \Dcal}\left[\Esp_{h' \sim \post  }  \ell(\h'(x),y)\mid  \ell(\h(x),y) = 1 \right] \\
         & \; + \underset{(x,y)\sim\Dcal}{\mathbb{P}} ( \ell(\h(x),y) = 0) \Esp_{(x,y) \sim \Pcal}\left[\Esp_{h' \sim \post  }  \ell(\h'(x),y)\mid  \ell(\h(x),y) = 0 \right],\\
         = & \;  \TripleC \underset{(x,y)\sim\Dcal}{\mathbb{P}} ( \ell(\h(x),y) = 1) + \TripleB \underset{(x,y)\sim\Dcal}{\mathbb{P}} ( \ell(\h(x),y) = 0),\\
         = & \; \TripleC \underset{(x,y)\sim\Dcal}{\mathbb{P}} ( \ell(\h(x),y) = 1) + \TripleB \left(1- \underset{(x,y)\sim\Dcal}{\mathbb{P}} ( \ell(\h(x),y) = 1)\right),\\
         = & \; (\TripleC - \TripleB) \underset{(x,y)\sim\Dcal}{\mathbb{P}} ( \ell(\h(x),y) = 1) + \TripleB,\\
         = & \; (\TripleC - \TripleB)  \Esp_{(x,y) \sim \Dcal} + \TripleB,\\
        = & \; (\TripleC - \TripleB)  \Hrisk{\Dcal} + \TripleB.
    \end{align*}

    The result is then obtained be rearranging the terms and and using the assumption: $\TripleC\neq\TripleB$.
\end{proof}



\subsection{Proof of~\Cref{lemma:det_upper_lower}}\label{app:proof:det_upper_lower}
\DetBounds*

\begin{proof}
We consider any distribution $\Dcal$, distribution $\post$ over $\Hcal$ and classifier $h\in\Hcal$ such that $\TripleC\neq\TripleB$. 
We apply the same steps as done in the proof of \Cref{proposition:triple_decomposition} to have
\begin{align*}
\GibbsRisk{\Dcal} = \; \TripleC \underset{(x,y)\sim\Dcal}{\mathbb{P}} ( \ell(\h(x),y) = 1) + \TripleB \left(1- \underset{(x,y)\sim\Dcal}{\mathbb{P}} ( \ell(\h(x),y) = 1)\right).
\end{align*}

\textbf{Upper bound.} 
We use the lower bounds of both $\TripleB$ and $\TripleC$, \ie, $\TripleB \ge \DownTripleB$ and $\TripleC \ge \DownTripleC$, to write
\begin{align*}
    \GibbsRisk{\Dcal} = &\; \; \TripleC \underset{(x,y)\sim\Dcal}{\mathbb{P}} ( \ell(\h(x),y) = 1) + \TripleB \left(1- \underset{(x,y)\sim\Dcal}{\mathbb{P}} ( \ell(\h(x),y) = 1)\right)\\
    \ge & \; \DownTripleC \underset{(x,y)\sim\Dcal}{\mathbb{P}} ( \ell(\h(x),y) = 1) + \DownTripleB \left(1- \underset{(x,y)\sim\Dcal}{\mathbb{P}} ( \ell(\h(x),y) = 1)\right)\\
    \ge & \;  \left(\DownTripleC  - \DownTripleB \right) \underset{(x,y)\sim\Dcal}{\mathbb{P}} ( \ell(\h(x),y) = 1) + \DownTripleB\\
    \ge & \;  \left(\DownTripleC  - \DownTripleB \right) \Hrisk{\Dcal}  + \DownTripleB.
\end{align*}
By rearranging the terms, we have,  with probability at least $1{-}\delta$ on $S \!\sim\! \Dcal^m$
\begin{align*}
\Hrisk{\Dcal} \le \dfrac{\GibbsRisk{\Dcal} - \DownTripleB}{\DownTripleC  - \DownTripleB} \le \dfrac{  \bound_{S}^{\uparrow}(\post, \prior, \delta, \Gibbs_{\post}) - \DownTripleB}{\DownTripleC  - \DownTripleB},
\end{align*}
where the last inequality comes from the lower bound on the Gibbs' risk $\GibbsRisk{\Dcal} \le \bound_{S}^{\uparrow}(\post, \prior, \delta, \Gibbs_{\post})$ that holds with probability $1{-}\delta$ on $S \!\sim\! \Dcal^m$ by assumption.

\textbf{Lower bound.}
We use the upper bounds of both $\TripleB$ and $\TripleC$, \ie $\TripleB \le \DownTripleB$ and $\TripleC \le \DownTripleC$, to write
\begin{align*}
    \GibbsRisk{\Dcal} = &\; \; \TripleC \underset{(x,y)\sim\Dcal}{\mathbb{P}} ( \ell(\h(x),y) = 1) + \TripleB \left(1- \underset{(x,y)\sim\Dcal}{\mathbb{P}} ( \ell(\h(x),y) = 1)\right)\\
    \le & \; \UpTripleC \underset{(x,y)\sim\Dcal}{\mathbb{P}} ( \ell(\h(x),y) = 1) + \UpTripleB \left(1- \underset{(x,y)\sim\Dcal}{\mathbb{P}} ( \ell(\h(x),y) = 1)\right)\\
    \le & \;  \left(\UpTripleC  - \UpTripleB \right) \underset{(x,y)\sim\Dcal}{\mathbb{P}} ( \ell(\h(x),y) = 1) + \UpTripleB\\
    \GibbsRisk{\Dcal} \le & \;  \left(\UpTripleC  - \UpTripleB \right) \Hrisk{\Dcal}  + \UpTripleB.
\end{align*}
By rearranging the terms, we have, with probability at least $1{-}\delta$ on $S \!\sim\! \Dcal^m$, we have
\begin{align*}
\Hrisk{\Dcal} \ge \dfrac{\GibbsRisk{\Dcal} - \UpTripleB}{\UpTripleC  - \UpTripleB} \ge \dfrac{  \bound_{S}^{\downarrow}(\post, \prior, \delta, \Gibbs_{\post}) - \UpTripleB}{\UpTripleC  - \UpTripleB}, 
\end{align*}
where the last inequality comes from the lower bound on the Gibbs' risk $\GibbsRisk{\Dcal} \ge \bound_{S}^{\downarrow}(\post, \prior, \delta, \Gibbs_{\post})$ that holds with probability $1{-}\delta$ on $S \!\sim\! \Dcal^m$ by assumption.
\end{proof}



\subsection{Proof of~\Cref{pr:bound_c_b}}\label{appendix:proof:bound_b_c}
\BoundsBC*
\begin{proof}
We briefly recall that $\TripleB$ and $\TripleC$ are respectively defined by
\begin{align*}
\TripleB  \defeq \! \underset{(x, y) \sim \Dcal}{\Esp}\left[\underset{\h^{\prime} \sim \post}{\Esp} \ell\left(\h^{\prime}(x), y\right) \;\middle|\; \ell(\h(x), y)=0 \right] \text{and} \;  \TripleC \defeq \! \underset{(x, y) \sim \Dcal}{\Esp}\left[\underset{\h^{\prime} \sim \post}{\Esp} \ell\left(\h^{\prime}(x), y\right) \;\middle|\; \ell(\h(x), y)=1\right].
\end{align*}

Let us denote $\alphabf^{\pm} = \sum_{i=1}^n \alpha_i\mathbf{I}\left\{f_i(x) \neq y\right\}$, the sum of the weights of the classifiers that misclassify $x$.
Thus, for any $(x, y)\in\Xcal\times\Ycal$, there exists $\mathbf{j}\in\{0,1\}^n$ such that $\alphabf^{\pm} = \mathbf{j}\cdot\alphabf$. 
Indeed, $\mathbf{j}$ can be seen as a vector indexing the base classifiers that misclassify an instance $x$. This representation will be used to draw the link between $\TripleB$ and $\TripleC$ and the partition problem algorithm.
We also recall that $\alphabf_1$ and $\alphabf_2$ are the result of the partition problem applied to~$\alphabf$, whereas \mbox{$\overset{\sim}{\alphabf} \!=\! \max\left(\sum_{\alpha\in\alphabf_1}\!\alpha, \sum_{\alpha\in\alphabf_2}\!\alpha\right)$}, and \mbox{$\overline{\alphabf} \!=\! \left|\sum_{\alpha\in\alphabf_1}\!\alpha\! -\! \sum_{\alpha\in\alphabf_2}\!\alpha\right|$}.
Now, we focus on the lower and upper bounds of $\TripleB$ and $\TripleC$ for the three forms of $\post$.

\textbf{Categorical distribution $\post=\cat(\alphabf)$.}
We first prove the upper bound of $\TripleBH$ when $\rho=\cat(\alphabf)$. Indeed, we have
\begingroup
\allowdisplaybreaks
\begin{align*}
    \TripleBH &= \underset{(x, y) \sim \Dcal}{\Esp}\left[\underset{\h^{\prime} \sim \cat(\alphabf)}{\Esp} \ell\left(\h^{\prime}(x), y\right) \;\middle|\; \ell(\h_{\alphabf}(x), y)=0 \right]\\
    &= \underset{(x, y) \sim \Dcal}{\Esp}\left[\underset{\h^{\prime} \sim \cat(\alphabf)}{\Esp} \mathbf{I}\big[\sign(h'(x)) \neq y\big] \;\middle|\; \mathbf{I}\big[\sign(h_{\alphabf}(x)) \neq y\big]=0 \right]\\
    &= \underset{(x, y)\sim \Dcal}{\mathbb{E}}\left[\alphabf^{\pm}~|~\alphabf^{\pm}< 0.5\right] &\langle\textup{\cite{Leblanc25}}\rangle\\
    &\leq \; \underset{(x, y)\in\Xcal\times\Ycal}{\max}\left[\alphabf^{\pm}~|~\alphabf^{\pm}< 0.5\right]\\
    &\leq \; \underset{\mathbf{i}\in\{0,1\}^n}{\max}\left[\mathbf{i}\cdot \alphabf~|~\mathbf{i}\cdot \alphabf< 0.5\right]\\
    &= \; \min\left(\sum_{\alpha\in\alphabf_1} \alpha, \sum_{\alpha\in \alphabf_2}\alpha\right) &\langle\textup{\cite{partition}}\rangle\\
    &= \;  1-\max\left(\sum_{\alpha\in\alphabf_1}\alpha, \sum_{\alpha\in\alphabf_2}\alpha\right)\\
    &= \; 1  - \Tilde{\alphabf}.
\end{align*}
\endgroup
Now, we focus on the lower bound of $\TripleCH$. We have
\begingroup
\allowdisplaybreaks
\begin{align*}
    \TripleCH &= \underset{(x, y)\sim \Dcal}{\mathbb{E}}\left[\alphabf^{\pm}~|~\alphabf^{\pm}\geq 0.5\right]&\langle\textup{\cite{Leblanc25}}\rangle\\
    &\geq \; \underset{(x, y)\in\Xcal\times\Ycal}{\min}\left[\alphabf^{\pm}~|~\alphabf^{\pm}\geq 0.5\right]\\
    &\geq \; \underset{\mathbf{i}\in\{0,1\}^n}{\min}\left[\mathbf{i}\cdot \alphabf~|~\mathbf{i}\cdot \alphabf\geq 0.5\right]\\
    &= \; \max\left(\sum_{\alpha\in\alphabf_1} \alpha, \sum_{\alpha\in \alphabf_2}\alpha\right)&\langle\textup{\cite{partition}}\rangle\\
    &= \;  \Tilde{\alphabf}.
\end{align*}
\endgroup

\textbf{Dirichlet distribution $\post = \dir(\alphabf)$.}
First of all, for the sake of clarity, we also denote $I_1$ the set of indexes of classifiers that err and $I_2$ the set of of classifier that predict correctly. 
Following \citet{zantedeschi2021learning}, using the aggregation property of the Dirichlet distribution, we can write 
$\left(\sum_{i \in I_1} w_i, \sum_{i\in I_2} w_i \right)$ is drawn from $\dir\left(\sum_{i \in I_1} \alpha_i, \sum_{i\in I_2} \alpha_i \right)$ which is no more than the $\Beta$ distribution with the same parameters.\\
Moreover, let us consider the random variable $S_w = \sum_{i \in I_1} w_i$ which represent the sum of the weights for classifiers that made mistakes. 
So $S_w \sim \Beta\left(\sum_{i \in I_1} \alpha_i, \sum_{i\in I_2} \alpha_i \right)$. 
Let us now compute the probability that the sum $S_w$ is greater than $0.5$, \ie, most of the classifiers wrongly classify the instance $x$. We have
\begin{align*}
\mathbb{P}[S_w \ge 0.5] = 1 - \mathbb{P}[S_w\le 0.5] = 1 - I_{0.5}\left(\sum_{i \in I_1} \alpha_i, \sum_{i\in I_2} \alpha_i \right) =   I_{0.5}\left(\sum_{i \in I_{2}} \alpha_i, \sum_{i\in I_1} \alpha_i \right) = I_{0.5}\bigg(\|\alphabf\|_{i} - \alphabf^{\pm}, \alphabf^{\pm} \bigg).
\end{align*}
We are now ready to prove the following upper bound $\TripleBH \leq I_{0.5}\left(\overset{\sim}{\alphabf}, ~\|\alphabf\|_1-\overset{\sim}{\alphabf}\right)$. 
We have
\begin{align*}
\TripleBH &=\underset{(x, y)\sim \Dcal}{\mathbb{E}}\left[I_{0.5}\left(\|\alphabf\|_1-\alphabf^{\pm}, \alphabf^{\pm}\right)\Big|\alphabf^{\pm}< \dfrac{\|\alphabf\|_1}{2}\right]~~~\langle\textup{ see \cite{Leblanc25} }\rangle\\
&\leq \underset{(x, y)\in\Xcal\times\Ycal}{\max}\left[I_{0.5}\left(\|\alphabf\|_1-\alphabf^{\pm}, \alphabf^{\pm}\right)\Big|\alphabf^{\pm}< \dfrac{\|\alphabf\|_1}{2}\right]\\
&\leq \underset{\mathbf{i}\in\{0,1\}^n}{\max}\left[I_{0.5}\left(\|\alphabf\|_1-\mathbf{i}\cdot\alphabf, \mathbf{i}\cdot\alphabf\right)\Big|\mathbf{i}\cdot\alphabf< \dfrac{\|\alphabf\|_1}{2}\right].
\end{align*}
Since $I_{0.5}(\cdot,\cdot)$ is decreasing in its first argument, and increasing in its second one, we have
\begin{align*}
\argmax_{\mathbf{i}\in\{0,1\}^n}\left[I_{0.5}\left(\|\alphabf\|_1-\mathbf{i}\cdot\alphabf, \mathbf{i}\cdot\alphabf\right)\Big|\,\mathbf{i}\cdot\alphabf<\dfrac{\|\alphabf\|_1}{2}\right] &= \underset{\mathbf{i}\in\{0,1\}^n}{\argmax}\left[\mathbf{i}\cdot\alphabf~|~\mathbf{i}\cdot\alphabf<\dfrac{\|\alphabf\|_1}{2}\right].\\
\end{align*}
Also, the function $\ibf \mapsto I_{0.5}\left(\|\alphabf\|_1- \ibf \cdot \alphabf, \ibf \cdot \alphabf\right)$ is maximized when $\ibf \cdot \alphabf$ is maximized.
Based on the development for the Categorical distribution, we have
\begin{align*}
\underset{\mathbf{i}\in\{0,1\}^n}{\max}\left[\mathbf{i}\cdot\alphabf~|~\mathbf{i}\cdot\alphabf<\dfrac{\|\alphabf\|_1}{2}\right] = \min\left(\sum_{\alpha\in\alphabf_1} \alpha, \sum_{\alpha\in \alphabf_2}\alpha\right) =\|\alphabf\|_1 - \Tilde{\alphabf}.
\end{align*}
Thus, we have
\begin{align*}
    &~\underset{\mathbf{i}\in\{0,1\}^n}{\max}\left[I_{0.5}\left(\|\alphabf\|_1-\mathbf{i}\cdot\alphabf, \mathbf{i}\cdot\alphabf\right)\Big|\mathbf{i}\cdot\alphabf< \dfrac{\|\alphabf\|_1}{2}\right] = I_{0.5}\left(\overset{\sim}{\alphabf}, \|\alphabf\|_1-\overset{\sim}{\alphabf}\right).
\end{align*}
Let us focus on the lower bound on $\TripleCH$. We have
\begin{align*}
    \TripleCH = & \; \underset{(x, y)\sim \Dcal}{\mathbb{E}}\left[I_{0.5}\left(\|\alphabf\|_1-\alphabf^{\pm}, \alphabf^{\pm}\right)\Big|\alphabf^{\pm}> \dfrac{\|\alphabf\|_1}{2}\right]~~~\langle\textup{ see \cite{Leblanc25} }\rangle\\
&\geq \underset{(x, y)\in\Xcal\times\Ycal}{\min}\left[I_{0.5}\left(\|\alphabf\|_1-\alphabf^{\pm}, \alphabf^{\pm}\right)\Big|\alphabf^{\pm}> \dfrac{\|\alphabf\|_1}{2}\right]\\
&\geq \underset{\mathbf{i}\in\{0,1\}^n}{\min}\left[I_{0.5}\left(\|\alphabf\|_1-\mathbf{i}\cdot\alphabf, \mathbf{i}\cdot\alphabf\right)\Big|\mathbf{i}\cdot\alphabf> \dfrac{\|\alphabf\|_1}{2}\right].
\end{align*}
Using the same properties on $\ibf \cdot \alphabf \mapsto I_{0.5}\left(\|\alphabf\|_1- \ibf \cdot \alphabf, \ibf \cdot \alphabf\right)$, we write
\begin{align*}
\underset{\mathbf{i}\in\{0,1\}^n}{\min}\left[\mathbf{i}\cdot\alphabf~|~\mathbf{i}\cdot\alphabf<\dfrac{\|\alphabf\|_1}{2}\right] = \max\left(\sum_{\alpha\in\alphabf_1} \alpha, \sum_{\alpha\in \alphabf_2}\alpha\right) = \Tilde{\alphabf}.
\end{align*}
Thus, using the same assumptions on $I_{0.5}$, we have
\begin{align*}
\TripleCH \ge I_{0.5}\left(\|\alphabf\|_1-\overset{\sim}{\alphabf}, \overset{\sim}{\alphabf}\right).
\end{align*}

\textbf{Normal distribution $\post=\gaus(\alphabf,\Ibf_{n\times n})$.}
First of all, using the Gaussian assumption, we have
\begin{align*}
\underset{h^{\prime}\sim \post}{\mathbb{P}}[\h^{\prime}(x)\neq y] = \underset{(w_1,\ldots,w_n) \sim \gaus(\alphabf,\Ibf_{n\times n})}{\mathbb{P}}\left[y\sum_{i=1}^n w_i f_i(x) < 0 \right].
\end{align*}
Let us denote $\wbf = (w_1,\ldots, w_n)$ and $\fbf = (f_1(x),\ldots, f_n(x))$, then we can write:
\begin{align*}
\underset{(w_1,\ldots,w_n) \sim \gaus(\alphabf,\Ibf_{n\times n})}{\mathbb{P}}\left[y\sum_{i=1}^n w_i f_i(x) < 0 \right] = \underset{(w_1,\ldots,w_n) \sim \gaus(\alphabf,\Ibf_{n\times n})}{\mathbb{P}}\left[ y \wbf \cdot \fbf < 0 \right],
\end{align*}
where $\wbf \cdot\fbf(x) \sim \gaus( \alphabf \cdot \fbf(x),\Vert \fbf(x)\Vert_2)$. 
Since both $f_i(x)$ and $y$ belong in $\{-1,+1\}$, we have $\Vert \fbf(x)\Vert_2 = \sqrt{n}$ and we can finally write
\begin{align*}
\underset{h^{\prime}\sim \post}{\mathbb{P}}[\h^{\prime}(x)\neq y] = \underset{(w_1,\ldots,w_n) \sim \gaus(\alphabf,\Ibf_{n\times n})}{\mathbb{P}}\left[y \wbf \cdot\fbf(x) < 0 \right] = \Tilde{\Phi}\left(-y\dfrac{ \alphabf \cdot \fbf(x) }{\sqrt{n}}\right) =  \Phi\left(y\dfrac{ \alphabf \cdot \fbf(x) }{\sqrt{n}}\right),
\end{align*}
where $\Tilde{\Phi}$ is the continuous density function of a centered and reduced Gaussian distribution with the relation:
\begin{align*}
\forall k \in \mathbb{R}\; \Tilde{\Phi}(-k) = \dfrac{1}{2}\left(1-\text{Erf}(k/\sqrt{2})\right) = \Phi(k).
\end{align*}

We are now able to prove the lower and upper bounds on $\TripleBH$ and $\TripleCH$.
We first prove the upper bound of $\TripleBH$.
\begingroup
\allowdisplaybreaks
\begin{align*}
\TripleBH  & = \underset{(x, y) \sim \Dcal}{\Esp}\left[\underset{\h^{\prime} \sim \post}{\Esp} \ell\left(\h^{\prime}(x), y\right) \;\middle|\; \ell(h_{\alphabf}(x), y)=0 \right],\\
&=\underset{(x, y)\sim \Dcal}{\mathbb{E}} \left[\Phi\left(y\frac{\alphabf\cdot\fbf(\xbf)}{\sqrt{n}}\right)~\Big|~y~\alphabf\cdot\fbf(\xbf) > 0\right],\\
&=\underset{(x, y)\sim \Dcal}{\mathbb{E}} \left[\Phi\left(\frac{|\alphabf\cdot\fbf(\xbf)|}{\sqrt{n}}\right)~\Big|~y~\alphabf\cdot\fbf(\xbf) > 0\right]&\langle\textup{ see \cite{Leblanc25} }\rangle\\
&\leq \underset{x\in\Xcal}{\max}~\Phi\left(\frac{|\alphabf\cdot\fbf(\xbf)|}{\sqrt{n}}\right)\\
&\leq \Phi\left(\frac{\underset{x\in\Xcal}{\min}|\alphabf\cdot\fbf(\xbf)|}{\sqrt{n}}\right)&\langle~\Phi\text{ is strictly decreasing. }\rangle\\
&\leq \Phi\left(\frac{\underset{\mathbf{i}\in\{-1,1\}^n}{\min}|\alphabf\cdot\ibf|}{\sqrt{n}}\right).
\end{align*}
\endgroup
Finally, note that 
\begin{align*}
\underset{\mathbf{i}\in\{-1,1\}^n}{\min}|\alphabf\cdot\ibf| = \min_{\alphabf_1,\alphabf_2}\left\{\big|\sum_{\alpha\in\alphabf_1} \alphabf - \sum_{\alpha\in\alphabf_2} \alphabf\Big|:\{\alphabf_1,\alphabf_2\}\textup{ is a partition of }\alphabf\right\},
\end{align*}
which corresponds to the objective of the partition problem to be minimized. Plugging that into the development yields the main result, we have
\begin{align*}
\TripleBH   \le \Phi \left(\dfrac{\overline{\alphabf}}{\sqrt{n}}\right).
\end{align*}
We now prove the lower bound of $\TripleBH$. We have
\begingroup
\allowdisplaybreaks
\begin{align*}
\TripleBH  & = \underset{(x, y) \sim \Dcal}{\Esp}\left[\underset{\h^{\prime} \sim \post}{\Esp} \ell\left(\h^{\prime}(x), y\right) \;\middle|\; \ell(h_{\alphabf}(x), y)=0 \right],\\
&=\underset{(x, y)\sim \Dcal}{\mathbb{E}} \left[\Phi\left(y\frac{\alphabf\cdot\fbf(\xbf)}{\sqrt{n}}\right)~\Big|~y~\alphabf\cdot\fbf(\xbf) > 0\right],\\
&=\underset{(x, y)\sim \Dcal}{\mathbb{E}} \left[\Phi\left(\frac{|\alphabf\cdot\fbf(\xbf)|}{\sqrt{n}}\right)~\Big|~y~\alphabf\cdot\fbf(\xbf) > 0\right]\\
&\geq \underset{x\in\Xcal}{\min}~\Phi\left(\frac{|\alphabf\cdot\fbf(\xbf)|}{\sqrt{n}}\right)\\
&\geq \Phi\left(\frac{\underset{x\in\Xcal}{\max}|\alphabf\cdot\fbf(\xbf)|}{\sqrt{n}}\right)&\langle~\Phi\text{ is strictly decreasing. }\rangle\\
&\geq \Phi\left(\frac{\underset{\mathbf{i}\in\{-1,1\}^n}{\max}|\alphabf\cdot\ibf|}{\sqrt{n}}\right)\\
&=\Phi\left(\frac{\|\alphabf\|_1}{\sqrt{n}}\right).
\end{align*}
\endgroup
We now prove the upper bound of $\TripleCH$. We have
\begingroup
\allowdisplaybreaks
\begin{align*}
 \TripleCH & =  \underset{(x, y) \sim \Dcal}{\Esp}\left[\underset{\h^{\prime} \sim \post}{\Esp} \ell\left(\h^{\prime}(x), y\right) \;\middle|\; \ell(\h(x), y)=1\right],\\
& =\underset{(x, y)\sim \Dcal}{\mathbb{E}} \left[\Phi\left(y\frac{\alphabf\cdot\fbf(\xbf)}{\sqrt{n}}\right)~\Big|~y~\alphabf\cdot\fbf(\xbf) < 0\right],\\
& = \underset{(x, y)\sim \Dcal}{\mathbb{E}} \left[\Phi\left(\frac{-|\alphabf\cdot\fbf(\xbf)|}{\sqrt{n}}\right)~\Big|~y~\alphabf\cdot\fbf(\xbf) \leq 0\right],\\
& =1-\underset{(x, y)\sim \Dcal}{\mathbb{E}} \left[\Phi\left(\frac{|\alphabf\cdot\fbf(\xbf)|}{\sqrt{n}}\right)~\Big|~y~\alphabf\cdot\fbf(\xbf) \leq 0\right]&\langle\textup{ see \cite{Leblanc25} }\rangle\\
&\leq 1-\underset{x\in\Xcal}{\min}~\Phi\left(\frac{|\alphabf\cdot\fbf(\xbf)|}{\sqrt{n}}\right)\\
&\leq  1-\Phi\left(\frac{\underset{x\in\Xcal}{\max}|\alphabf\cdot\fbf(\xbf)|}{\sqrt{n}}\right)&\langle~\Phi\text{ is strictly decreasing. }\rangle\\
&=  1-\Phi\left(\frac{\|\alphabf\|_1}{\sqrt{n}}\right).
\end{align*}
\endgroup
We now prove the lower bound of $\TripleCH$. We have
\begingroup
\allowdisplaybreaks
\begin{align*}
 \TripleCH & =  \underset{(x, y) \sim \Dcal}{\Esp}\left[\underset{\h^{\prime} \sim \post}{\Esp} \ell\left(\h^{\prime}(x), y\right) \;\middle|\; \ell(\h(x), y)=1\right],\\
 & =\underset{(x, y)\sim \Dcal}{\mathbb{E}} \left[\Phi\left(y\frac{\alphabf\cdot\fbf(\xbf)}{\sqrt{n}}\right)~\Big|~y~\alphabf\cdot\fbf(\xbf) < 0\right],\\
& = \underset{(x, y)\sim \Dcal}{\mathbb{E}} \left[\Phi\left(\frac{-|\alphabf\cdot\fbf(\xbf)|}{\sqrt{n}}\right)~\Big|~y~\alphabf\cdot\fbf(\xbf) \leq 0\right],\\
 &=1-\underset{(x, y)\sim \Dcal}{\mathbb{E}} \left[\Phi\left(\frac{|\alphabf\cdot\fbf(\xbf)|}{\sqrt{n}}\right)~\Big|~y~\alphabf\cdot\fbf(\xbf) \leq 0\right]\\
&\leq 1-\underset{x\in\Xcal}{\max}~\Phi\left(\frac{|\alphabf\cdot\fbf(\xbf)|}{\sqrt{n}}\right)\\
&\leq  1-\Phi\left(\frac{\underset{x\in\Xcal}{\min}|\alphabf\cdot\fbf(\xbf)|}{\sqrt{n}}\right)&\langle~\Phi\text{ is strictly decreasing. }\rangle\\
&=  1-\Phi\left(\frac{\overline{\alphabf}}{\sqrt{n}}\right).
\end{align*}
\endgroup
\end{proof}

\subsection{Trade-off upper bounds proofs}\label{app:trade-off_proofs}

\begin{theorem}\label{theorem:trade-off_randomized}
For any distribution $\Dcal$, any hypothesis set~$\Hcal$, any prior $\prior$ on $\Hcal$, and $\delta\!\in\! (0,1]$, $\lambda > 0$ with probability at least $1{-}\delta$ on the random choice $S{\sim} \Dcal^m$ we have for any $\post$ over $\Hcal$ and any $\h \!\in\! \Hcal$
\begin{align*}
 (1-\lambda)\GibbsRisk{\Dcal}+\lambda\FairRisk{\Dcal}{G_\post}  \leq &(1{-}\lambda)\ \bound_{S}^{\uparrow}(\post, \prior, \tfrac{\delta}{2}, \Gibbs_{\post}) + \\
& \nonumber \lambda\max\big\{\bound_{S_{|a\!}}^{\uparrow}(\post, \prior, \tfrac{\delta}{8}, \Gibbs_\post) \!-\!\bound_{S_{|b\!}}^{\downarrow}(\post, \prior, \tfrac{\delta}{8}, \Gibbs_{\post}),\\
& \phantom{+ \lambda\max\big\{\ \ }\bound_{S_{|b\!}}^{\uparrow}(\post, \prior, \tfrac{\delta}{8}, \Gibbs_{\post}) \!-\! \bound_{S_{|a\!}}^{\downarrow}(\post, \prior, \tfrac{\delta}{8}, \Gibbs_{\post})\big\}\,,
\end{align*}
\end{theorem}
\begin{proof}
We apply~\Cref{theorem:classical_kl} to upper-bound $ (1-\lambda)\GibbsRisk{\Dcal}$ and~\Cref{theorem:gibbs_fairness_bound} to upper-bound $\lambda\FairRisk{\Dcal}{G_\post}$, both with probabilty with probabilty $1-\tfrac{\delta}{2}$. Combining them via union bound yields the stated result
\end{proof}

\begin{theorem}\label{theorem:trade-off_deterministic}
For any distribution $\Dcal$, any hypothesis set~$\Hcal$, any prior $\prior$ on $\Hcal$, and $\delta\!\in\! (0,1]$, $\lambda > 0$ with probability at least $1{-}\delta$ on the random choice $S{\sim} \Dcal^m$ we have for any $\post$ over $\Hcal$ and any $\h \!\in\! \Hcal$
\begin{align*}
 (1-\lambda)\Risk(h) +\lambda\FairRisk{\Dcal}{h}  \leq & (1{-}\lambda)\ \BoundMV_{S_{}}^{\uparrow}(\post, \prior, \tfrac{\delta}{2}, h)\nonumber +\\ 
 &\lambda\max\big\{\BoundMV_{S_{|\GrOne}\!}^{\uparrow}(\post, \prior, \tfrac{\delta}{8}, h) \!-\! \BoundMV_{S_{|\GrTwo}\!}^{\downarrow}(\post, \prior, \tfrac{\delta}{8}, h),\BoundMV_{S_{|\GrTwo}\!}^{\uparrow}(\post, \prior, \tfrac{\delta}{8}, h) \!-\! \BoundMV_{S_{|\GrOne}\!}^{\downarrow}(\post, \prior, \tfrac{\delta}{8},h)
\big\},
\end{align*}
\end{theorem}
\begin{proof}
We apply~\Cref{lemma:det_upper_lower} to upper-bound $(1-\lambda)\Risk(\halpha)$ and~\Cref{theorem:deterministic_abs_dif} to upper-bound $\lambda\FairRisk{\Dcal}{\halpha}$, both with probabilty with probabilty $1-\tfrac{\delta}{2}$. Combining them via union bound yields the stated result
\end{proof}

\section{Experimentations}\label{appendix:experimentations}
We used an NVIDIA GeForce RTX 2080 Ti graphics card for the experiments. We used a batch size equal to 1024, and a learning rate equal to 0.1 with a scheduler reducing this parameter by a factor of 10 with an epoch patience of 2. The maximal number of epochs is set to 100, and patience is set to 25 for performing early stopping. For each distribution posterior family (Categorical, Dirichlet and Gaussian) we use a uniform prior from the same family as the prior. Let $\alphabf_\prior$ be the parameter vector corresponding to the prior distribution $\prior$. We have for the three possible distribution family :
\begin{enumerate}
    \item Categorical $\alpha^{\prior}_{i} = \frac{1}{n}$ and $\prior = \cat(\alphabf^\prior)$
    \item Dirichlet $\alpha^{\prior}_{i} =1$ and $\prior = \dir(\alphabf^\prior)$
    \item  Gaussian $\alpha^{\prior}_{i} = 0$ and $\prior = \gaus(\alphabf^\prior, \Ibf_{n\times n})$
\end{enumerate}

\subsection{Datasets overview}\label{app:datasets_overview}

\begin{table}[H]
\centering
\caption{Overview of the datasets used in the experiments and their overall imbalance, where $m$ refers to the number of observations and $d$ the number of features.}
\label{tab:datasets_overview1}
\begin{tabular}{|c|c|c|c|}
\hline
Dataset & $m$ & $d$ & $\Pbb_S(y=+1)$ \\
\hline
ADULT  & 48842 & 6  & 0.24 \\
COMPAS & 5278  & 9  & 0.53 \\
GERMAN & 1000  & 10 & 0.70 \\
MEPS   & 15839 & 41 & 0.82 \\
\hline
\end{tabular}
\end{table}

\begin{table}[H]
    \centering
    \caption{An overview of the various datasets used in the experiments, with the sensitive attributes available for each and the balance according to this attribute, where $m$ refers to the number of observations and $d$ the number of features.}
    \label{tab:datasets_overview2}
    \begin{tabular}{|cc|cc|cc|cc|}
    \hline
    Dataset & Sens. attr. & $m$ & $d$ & $\Pbb_S(A=0)$ & $\Pbb_S(y=+1~|~A=0)$ & $\Pbb_S(A=1)$ & $\Pbb_S(y=+1~|~A=1)$ \\
    \hline
    \multirow{2}{*}{ADULT} & RACE & \multirow{2}{*}{48842} & \multirow{2}{*}{6} & 0.14 & 0.15 & 0.86 & 0.25\\ 
     & GENDER &  &  & 0.33 & 0.11 & 0.67 & 0.30  \\ 
    \hline
    \multirow{2}{*}{COMPAS} & RACE & \multirow{2}{*}{5278} & \multirow{2}{*}{9} & 0.60 & 0.48 & 0.40 & 0.61\\ 
     & GENDER &  &  & 0.80 & 0.50 & 0.20 & 0.64 \\ 
    \hline
    \multirow{2}{*}{GERMAN} & AGE & \multirow{2}{*}{1000} & \multirow{2}{*}{10} & 0.15 & 0.59 & 0.85 & 0.72 \\
     & GENDER &  &  & 0.31 & 0.65 & 0.69 & 0.72 \\ 
    \hline
    \multirow{2}{*}{MEPS} & RACE & \multirow{2}{*}{15839} & \multirow{2}{*}{41} & 0.64 & 0.87 & 0.36 & 0.74 \\
     & GENDER &  &  & 0.52 & 0.79 & 0.48 & 0.87 \\ 
    \hline
    \end{tabular}
\end{table}

\begin{table}[H]
\centering
\caption{Prediction task and sensitive attributes. For each dataset, we specify the favorable outcome ($y=+1$) and the sensitive attributes.
When a sensitive attribute is non-binary, it is binarized by contrasting the privileged group
against the rest (\eg, \emph{White} vs.\ \emph{Non-White}).}
\label{tab:tasks_sensitive}
\begin{tabular}{|c|c|c|c|}
\hline
Dataset & $y=+1$ definition & Attribute & Values \\
\hline

\multirow{2}{*}{ADULT} & \multirow{2}{*}{Annual income $>50$k\$} & Gender & Male / Female \\
& & Race   & White / Non-White \\

\hline

\multirow{2}{*}{COMPAS}
& \multirow{2}{*}{No recidivism after 2 years}
& Gender & Male / Female \\
& & Race   & White / Non-White \\

\hline

\multirow{2}{*}{GERMAN}
& \multirow{2}{*}{Good credit}
& Gender & Male / Female \\
& & Age    & $\geq 25$ years \\

\hline

\multirow{2}{*}{MEPS}
& \multirow{2}{*}{Fewer than $10$ medical visits}
& Gender & Male / Female \\
& & Race   & White / Non-White \\

\hline
\end{tabular}
\end{table}

\subsection{Detailed results}\label{app:experimentations_detailed}

In~\Cref{tab:results_bounds,tab:results_measure}, we report the test performance and corresponding generalization bounds obtained when training under each fairness constraint (DP, EO, and EOP).
Additionally,~\Cref{fig:plot_EO} provides a visual illustration of the results for models trained with the EO constraint, while~\Cref{fig:plot_EOP} presents the analogous results for EOP.\\
The results are consistent with what we discussed in~\Cref{sec:expe} and confirm the meaningfulness of the obtained risk and fairness bounds.

\begin{table}[ht]
    \centering
    \caption{Average bound value, over 5 runs, and standard deviation of the risk and fairness for each measure considered for training. For EO the O2 method cannot be directly generalized because the fairness risk is a linear combination of risks; therefore, we do not report it.}
    \label{tab:results_bounds}
  \resizebox{\linewidth}{!}{\begin{tabular}{|ccc|cc|cc|cc|}
    \hline
    Dataset & Sens. attr. & Type & Risk & DP & Risk & EO & Risk & EOP \\
    \hline
    \multirow{8}{*}{ADULT} & \multirow{4}{*}{GENDER} & Ours - det. & 21.45 $\pm$ 0.07 & 5.33 $\pm$ 0.05 & 21.45 $\pm$ 0.07 & 3.87 $\pm$ 0.08 & 21.45 $\pm$ 0.07 & 10.95 $\pm$ 0.27\\
 &  & Ours - sto. & 43.91 $\pm$ 0.18 & 4.6 $\pm$ 0.04 & 21.69 $\pm$ 0.07 & 6.79 $\pm$ 0.08 & 43.91 $\pm$ 0.18 & 9.65 $\pm$ 0.12\\
 &  & O1 - sto. & 43.89 $\pm$ 0.18 & 4.5 $\pm$ 0.04 & 21.44 $\pm$ 0.07 & 3.87 $\pm$ 0.08 & 43.89 $\pm$ 0.18 & 9.59 $\pm$ 0.12\\
 &  & O2 - sto. & 50.79 $\pm$ 0.01 & 3.5 $\pm$ 0.01 & N/A & N/A & 47.22 $\pm$ 0.06 & 8.5 $\pm$ 0.09\\
\cline{2-9}
 & \multirow{4}{*}{OTHER} & Ours - det. & 21.45 $\pm$ 0.07 & 4.07 $\pm$ 0.09 & 21.45 $\pm$ 0.07 & 3.76 $\pm$ 0.15 & 21.45 $\pm$ 0.07 & 10.61 $\pm$ 0.57\\
 &  & Ours - sto. & 21.69 $\pm$ 0.07 & 6.99 $\pm$ 0.1 & 21.69 $\pm$ 0.07 & 7.44 $\pm$ 0.15 & 21.69 $\pm$ 0.07 & 12.31 $\pm$ 0.53\\
 &  & O1 - sto. & 21.44 $\pm$ 0.07 & 4.06 $\pm$ 0.09 & 21.44 $\pm$ 0.07 & 3.76 $\pm$ 0.15 & 21.44 $\pm$ 0.07 & 10.6 $\pm$ 0.57\\
 &  & O2 - sto. & 50.77 $\pm$ 0.01 & 4.41 $\pm$ 0.02 & N/A & N/A & 29.4 $\pm$ 0.13 & 10.35 $\pm$ 0.21\\
\hline
\multirow{8}{*}{COMPAS} & \multirow{4}{*}{GENDER} & Ours - det. & 40.34 $\pm$ 0.45 & 24.15 $\pm$ 0.43 & 40.34 $\pm$ 0.45 & 25.86 $\pm$ 0.58 & 40.34 $\pm$ 0.45 & 20.73 $\pm$ 0.55\\
 &  & Ours - sto. & 52.59 $\pm$ 0.17 & 12.97 $\pm$ 1.32 & 51.22 $\pm$ 0.15 & 18.28 $\pm$ 0.38 & 51.22 $\pm$ 0.15 & 16.28 $\pm$ 0.43\\
 &  & O1 - sto. & 52.58 $\pm$ 0.17 & 12.91 $\pm$ 1.31 & 51.22 $\pm$ 0.15 & 17.97 $\pm$ 0.36 & 51.22 $\pm$ 0.15 & 15.99 $\pm$ 0.4\\
 &  & O2 - sto. & 53.75 $\pm$ 0.0 & 13.6 $\pm$ 0.01 & N/A & N/A & 52.09 $\pm$ 0.04 & 16.9 $\pm$ 0.57\\
\cline{2-9}
 & \multirow{4}{*}{OTHER} & Ours - det. & 40.34 $\pm$ 0.45 & 27.05 $\pm$ 0.72 & 40.34 $\pm$ 0.45 & 28.44 $\pm$ 0.6 & 40.34 $\pm$ 0.45 & 23.81 $\pm$ 1.19\\
 &  & Ours - sto. & 52.63 $\pm$ 0.13 & 11.51 $\pm$ 1.22 & 52.7 $\pm$ 0.11 & 16.17 $\pm$ 1.16 & 52.7 $\pm$ 0.11 & 15.18 $\pm$ 0.91\\
 &  & O1 - sto. & 52.63 $\pm$ 0.13 & 11.45 $\pm$ 1.21 & 52.69 $\pm$ 0.11 & 16.03 $\pm$ 1.14 & 52.69 $\pm$ 0.11 & 15.05 $\pm$ 0.89\\
 &  & O2 - sto. & 52.89 $\pm$ 0.04 & 10.73 $\pm$ 0.08 & N/A & N/A & 52.29 $\pm$ 0.06 & 14.18 $\pm$ 0.15\\
\hline
\multirow{8}{*}{GERMAN} & \multirow{4}{*}{GENDER} & Ours - det. & 37.02 $\pm$ 0.45 & 9.93 $\pm$ 1.67 & 37.02 $\pm$ 0.45 & 14.38 $\pm$ 2.51 & 37.02 $\pm$ 0.45 & 8.2 $\pm$ 2.18\\
 &  & Ours - sto. & 37.48 $\pm$ 0.42 & 27.77 $\pm$ 0.63 & 37.48 $\pm$ 0.42 & 40.08 $\pm$ 0.77 & 37.48 $\pm$ 0.42 & 33.65 $\pm$ 0.32\\
 &  & O1 - sto. & 37.02 $\pm$ 0.45 & 9.92 $\pm$ 1.67 & 37.02 $\pm$ 0.45 & 14.38 $\pm$ 2.51 & 37.02 $\pm$ 0.45 & 8.2 $\pm$ 2.18\\
 &  & O2 - sto. & 57.33 $\pm$ 0.0 & 22.14 $\pm$ 0.01 & N/A & N/A & 57.33 $\pm$ 0.0 & 26.57 $\pm$ 0.12\\
\cline{2-9}
 & \multirow{4}{*}{OTHER} & Ours - det. & 37.02 $\pm$ 0.45 & 19.15 $\pm$ 5.13 & 37.02 $\pm$ 0.45 & 26.75 $\pm$ 5.24 & 37.02 $\pm$ 0.45 & 20.03 $\pm$ 3.67\\
 &  & Ours - sto. & 37.48 $\pm$ 0.42 & 33.86 $\pm$ 2.45 & 37.48 $\pm$ 0.42 & 47.81 $\pm$ 1.88 & 37.48 $\pm$ 0.42 & 40.96 $\pm$ 0.37\\
 &  & O1 - sto. & 37.02 $\pm$ 0.45 & 19.15 $\pm$ 5.13 & 37.02 $\pm$ 0.45 & 26.75 $\pm$ 5.24 & 37.02 $\pm$ 0.45 & 20.02 $\pm$ 3.67\\
 &  & O2 - sto. & 57.33 $\pm$ 0.0 & 26.55 $\pm$ 0.0 & N/A & N/A & 57.33 $\pm$ 0.0 & 33.18 $\pm$ 0.35\\
\hline
\multirow{8}{*}{MEPS} & \multirow{4}{*}{GENDER} & Ours - det. & 21.9 $\pm$ 1.95 & 39.8 $\pm$ 12.4 & 21.9 $\pm$ 1.95 & 39.12 $\pm$ 12.58 & 21.9 $\pm$ 1.95 & 35.4 $\pm$ 12.02\\
 &  & Ours - sto. & 19.81 $\pm$ 0.05 & 12.35 $\pm$ 0.94 & 19.31 $\pm$ 0.04 & 13.48 $\pm$ 0.34 & 19.31 $\pm$ 0.04 & 11.15 $\pm$ 0.28\\
 &  & O1 - sto. & 19.01 $\pm$ 0.07 & 6.86 $\pm$ 0.94 & 18.57 $\pm$ 0.04 & 7.22 $\pm$ 0.6 & 18.57 $\pm$ 0.04 & 4.73 $\pm$ 0.49\\
 &  & O2 - sto. & 51.6 $\pm$ 0.01 & 5.66 $\pm$ 0.02 & N/A & N/A & 51.6 $\pm$ 0.01 & 6.19 $\pm$ 0.0\\
\cline{2-9}
 & \multirow{4}{*}{OTHER} & Ours - det. & 21.46 $\pm$ 2.16 & 38.85 $\pm$ 13.76 & 21.46 $\pm$ 2.16 & 36.76 $\pm$ 13.73 & 21.46 $\pm$ 2.16 & 34.27 $\pm$ 13.36\\
 &  & Ours - sto. & 47.53 $\pm$ 2.05 & 8.0 $\pm$ 0.88 & 17.65 $\pm$ 0.08 & 11.95 $\pm$ 0.24 & 17.65 $\pm$ 0.08 & 10.72 $\pm$ 0.13\\
 &  & O1 - sto. & 47.52 $\pm$ 2.06 & 7.9 $\pm$ 0.82 & 17.04 $\pm$ 0.08 & 8.35 $\pm$ 0.25 & 17.04 $\pm$ 0.08 & 6.6 $\pm$ 0.14\\
 &  & O2 - sto. & 51.6 $\pm$ 0.01 & 5.83 $\pm$ 0.02 & N/A & N/A & 51.6 $\pm$ 0.01 & 6.5 $\pm$ 0.01\\
    \hline
    \end{tabular}}
\end{table}

\begin{table}[ht]
    \centering
    \caption{Average test value, over 5 runs, and standard deviation of the risk and fairness for each measure considered for training. For EO the O2 method cannot be directly generalized because the fairness risk is a linear combination of risks; therefore, we do not report it.}
    \label{tab:results_measure}
  \resizebox{\linewidth}{!}{\begin{tabular}{|ccc|cc|cc|cc|}
    \hline
    Dataset & Sens. attr. & Type & Risk & DP & Risk & EO & Risk & EOP \\
    \hline
    \multirow{8}{*}{ADULT} & \multirow{4}{*}{GENDER} & Ours - det. & 20.27 $\pm$ 0.28 & 3.75 $\pm$ 0.21 & 20.27 $\pm$ 0.28 & 0.76 $\pm$ 0.34 & 20.27 $\pm$ 0.28 & 1.75 $\pm$ 1.07\\
 &  & Ours - sto. & 42.64 $\pm$ 0.17 & 0.93 $\pm$ 0.07 & 20.27 $\pm$ 0.28 & 0.76 $\pm$ 0.34 & 42.64 $\pm$ 0.17 & 0.43 $\pm$ 0.26\\
 &  & O1 - sto. & 42.64 $\pm$ 0.17 & 0.93 $\pm$ 0.07 & 20.27 $\pm$ 0.28 & 0.76 $\pm$ 0.34 & 42.64 $\pm$ 0.17 & 0.43 $\pm$ 0.26\\
 &  & O2 - sto. & 49.6 $\pm$ 0.01 & 0.01 $\pm$ 0.01 & N/A & N/A & 46.02 $\pm$ 0.08 & 0.08 $\pm$ 0.07\\
\cline{2-9}
 & \multirow{4}{*}{OTHER} & Ours - det. & 20.27 $\pm$ 0.28 & 1.49 $\pm$ 0.38 & 20.27 $\pm$ 0.28 & 0.83 $\pm$ 0.48 & 20.27 $\pm$ 0.28 & 2.89 $\pm$ 2.03\\
 &  & Ours - sto. & 20.27 $\pm$ 0.28 & 1.49 $\pm$ 0.38 & 20.27 $\pm$ 0.28 & 0.83 $\pm$ 0.48 & 20.27 $\pm$ 0.28 & 2.89 $\pm$ 2.03\\
 &  & O1 - sto. & 20.27 $\pm$ 0.28 & 1.49 $\pm$ 0.38 & 20.27 $\pm$ 0.28 & 0.83 $\pm$ 0.48 & 20.27 $\pm$ 0.28 & 2.89 $\pm$ 2.03\\
 &  & O2 - sto. & 49.58 $\pm$ 0.01 & 0.09 $\pm$ 0.01 & N/A & N/A & 28.13 $\pm$ 0.19 & 0.81 $\pm$ 0.42\\
\hline
\multirow{8}{*}{COMPAS} & \multirow{4}{*}{GENDER} & Ours - det. & 36.34 $\pm$ 1.77 & 14.92 $\pm$ 1.73 & 36.34 $\pm$ 1.77 & 12.45 $\pm$ 2.63 & 36.34 $\pm$ 1.77 & 8.68 $\pm$ 3.14\\
 &  & Ours - sto. & 49.41 $\pm$ 0.23 & 2.26 $\pm$ 1.32 & 47.8 $\pm$ 0.29 & 2.01 $\pm$ 0.46 & 47.8 $\pm$ 0.29 & 1.39 $\pm$ 0.46\\
 &  & O1 - sto. & 49.41 $\pm$ 0.23 & 2.26 $\pm$ 1.32 & 47.8 $\pm$ 0.29 & 2.01 $\pm$ 0.46 & 47.8 $\pm$ 0.29 & 1.39 $\pm$ 0.46\\
 &  & O2 - sto. & 49.82 $\pm$ 0.01 & 0.09 $\pm$ 0.03 & N/A & N/A & 48.41 $\pm$ 0.11 & 0.5 $\pm$ 0.21\\
\cline{2-9}
 & \multirow{4}{*}{OTHER} & Ours - det. & 36.34 $\pm$ 1.77 & 17.01 $\pm$ 2.98 & 36.34 $\pm$ 1.77 & 13.98 $\pm$ 2.53 & 36.34 $\pm$ 1.77 & 11.96 $\pm$ 5.12\\
 &  & Ours - sto. & 49.46 $\pm$ 0.22 & 2.11 $\pm$ 1.42 & 49.53 $\pm$ 0.17 & 2.59 $\pm$ 1.36 & 49.53 $\pm$ 0.17 & 2.57 $\pm$ 1.36\\
 &  & O1 - sto. & 49.46 $\pm$ 0.22 & 2.11 $\pm$ 1.42 & 49.53 $\pm$ 0.17 & 2.59 $\pm$ 1.36 & 49.53 $\pm$ 0.17 & 2.57 $\pm$ 1.36\\
 &  & O2 - sto. & 49.19 $\pm$ 0.04 & 0.9 $\pm$ 0.12 & N/A & N/A & 48.65 $\pm$ 0.07 & 1.62 $\pm$ 0.2\\
\hline
\multirow{8}{*}{GERMAN} & \multirow{4}{*}{GENDER} & Ours - det. & 29.4 $\pm$ 2.04 & 2.89 $\pm$ 1.61 & 29.4 $\pm$ 2.04 & 3.4 $\pm$ 1.58 & 29.4 $\pm$ 2.04 & 2.83 $\pm$ 2.05\\
 &  & Ours - sto. & 29.4 $\pm$ 2.04 & 2.89 $\pm$ 1.61 & 29.4 $\pm$ 2.04 & 3.4 $\pm$ 1.58 & 29.4 $\pm$ 2.04 & 2.83 $\pm$ 2.05\\
 &  & O1 - sto. & 29.4 $\pm$ 2.04 & 2.89 $\pm$ 1.61 & 29.4 $\pm$ 2.04 & 3.4 $\pm$ 1.58 & 29.4 $\pm$ 2.04 & 2.83 $\pm$ 2.05\\
 &  & O2 - sto. & 49.78 $\pm$ 0.01 & 0.11 $\pm$ 0.02 & N/A & N/A & 49.78 $\pm$ 0.01 & 0.1 $\pm$ 0.01\\
\cline{2-9}
 & \multirow{4}{*}{OTHER} & Ours - det. & 29.4 $\pm$ 2.04 & 3.93 $\pm$ 4.27 & 29.4 $\pm$ 2.04 & 5.15 $\pm$ 3.92 & 29.4 $\pm$ 2.04 & 3.88 $\pm$ 3.77\\
 &  & Ours - sto. & 29.4 $\pm$ 2.04 & 3.93 $\pm$ 4.27 & 29.4 $\pm$ 2.04 & 5.15 $\pm$ 3.92 & 29.4 $\pm$ 2.04 & 3.88 $\pm$ 3.77\\
 &  & O1 - sto. & 29.4 $\pm$ 2.04 & 3.93 $\pm$ 4.27 & 29.4 $\pm$ 2.04 & 5.15 $\pm$ 3.92 & 29.4 $\pm$ 2.04 & 3.88 $\pm$ 3.77\\
 &  & O2 - sto. & 49.78 $\pm$ 0.01 & 0.19 $\pm$ 0.01 & N/A & N/A & 49.78 $\pm$ 0.01 & 0.18 $\pm$ 0.02\\
\hline
\multirow{8}{*}{MEPS} & \multirow{4}{*}{GENDER} & Ours - det. & 15.11 $\pm$ 0.33 & 4.47 $\pm$ 0.8 & 15.11 $\pm$ 0.33 & 2.53 $\pm$ 0.73 & 15.11 $\pm$ 0.33 & 2.29 $\pm$ 0.81\\
 &  & Ours - sto. & 16.31 $\pm$ 0.16 & 2.42 $\pm$ 0.46 & 16.14 $\pm$ 0.16 & 1.5 $\pm$ 0.47 & 16.14 $\pm$ 0.16 & 1.45 $\pm$ 0.36\\
 &  & O1 - sto. & 16.31 $\pm$ 0.16 & 2.42 $\pm$ 0.46 & 16.14 $\pm$ 0.16 & 1.5 $\pm$ 0.47 & 16.14 $\pm$ 0.16 & 1.45 $\pm$ 0.36\\
 &  & O2 - sto. & 49.56 $\pm$ 0.01 & 0.02 $\pm$ 0.0 & N/A & N/A & 49.56 $\pm$ 0.01 & 0.03 $\pm$ 0.0\\
\cline{2-9}
 & \multirow{4}{*}{OTHER} & Ours - det. & 15.11 $\pm$ 0.33 & 6.34 $\pm$ 0.65 & 15.11 $\pm$ 0.33 & 2.57 $\pm$ 0.33 & 15.11 $\pm$ 0.33 & 2.2 $\pm$ 0.39\\
 &  & Ours - sto. & 45.28 $\pm$ 2.09 & 1.21 $\pm$ 0.66 & 15.11 $\pm$ 0.33 & 2.56 $\pm$ 0.33 & 15.11 $\pm$ 0.33 & 2.19 $\pm$ 0.39\\
 &  & O1 - sto. & 45.28 $\pm$ 2.09 & 1.21 $\pm$ 0.66 & 15.11 $\pm$ 0.33 & 2.56 $\pm$ 0.33 & 15.11 $\pm$ 0.33 & 2.19 $\pm$ 0.39\\
 &  & O2 - sto. & 49.56 $\pm$ 0.01 & 0.02 $\pm$ 0.0 & N/A & N/A & 49.56 $\pm$ 0.01 & 0.02 $\pm$ 0.01\\
    \hline
    \end{tabular}}
\end{table}

\begin{figure}[ht]
    \centering
    \begin{subfigure}{0.9\linewidth}
        \centering
        \includegraphics[width=\linewidth]{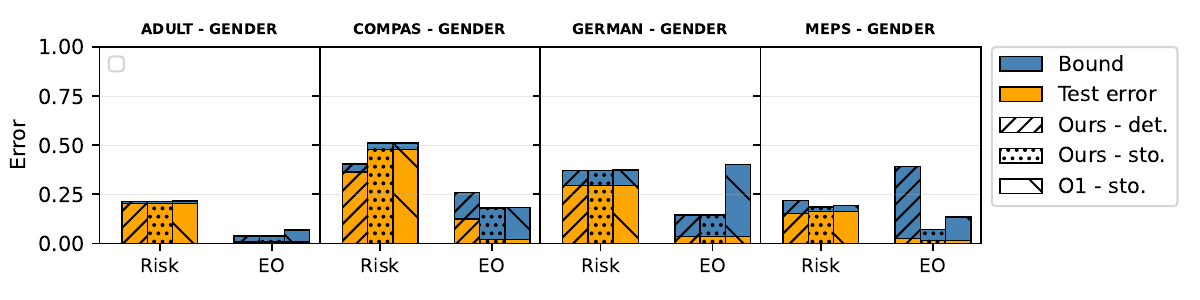}
        \caption{Sensitive attribute: Gender.}
        \label{fig:plot_EO_gender}
    \end{subfigure}

    \vspace{0.5em}

    \begin{subfigure}{0.9\linewidth}
        \centering
        \includegraphics[width=\linewidth]{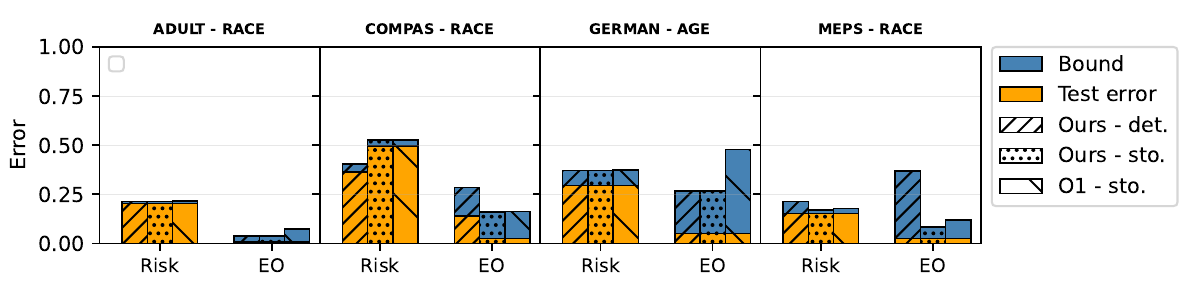}
        \caption{Sensitive attribute: Race or Age (depending on availability).}
        \label{fig:plot_EO_other}
    \end{subfigure}

    \caption{Test error and generalization bound of a stochastic majority vote classifier and its deterministic counterpart for Equalized Odds (EO). To compute the Equalized Odds risk, we replace $\Pbb[y=+1]$ and $\Pbb[y=0]$ by their empirical estimate. In this case, the O2 method cannot be directly generalized because the fairness risk is a linear combination of risks; therefore, we do not report it.}
    \label{fig:plot_EO}
\end{figure}

\begin{figure}[ht]
    \centering
    \begin{subfigure}{0.9\linewidth}
        \centering
        \includegraphics[width=\linewidth]{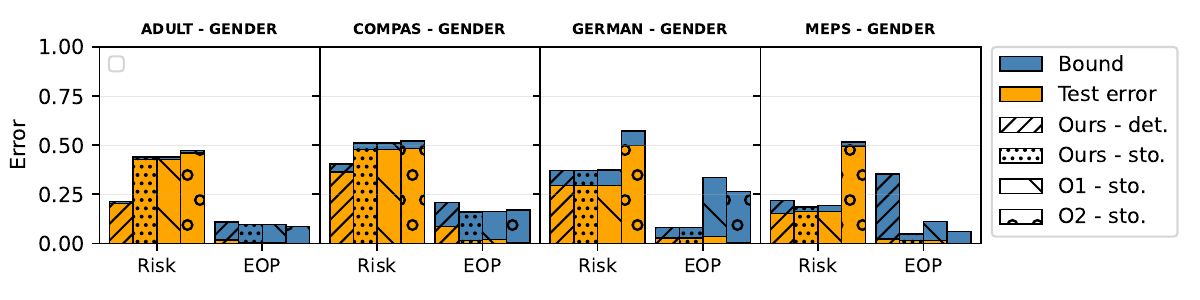}
        \caption{Sensitive attribute: Gender.}
        \label{fig:plot_EOP_gender}
    \end{subfigure}

    \vspace{0.5em}

    \begin{subfigure}{0.9\linewidth}
        \centering
        \includegraphics[width=\linewidth]{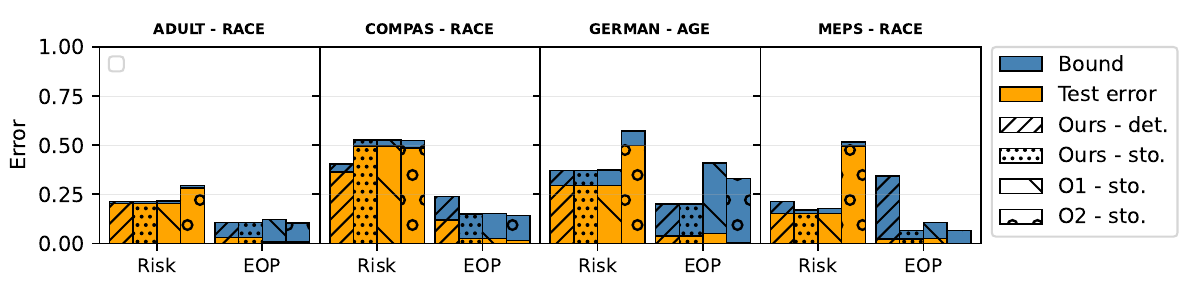}
        \caption{Sensitive attribute: Race or Age (depending on availability).}
        \label{fig:plot_EOP_other}
    \end{subfigure}

    \caption{Test error and generalization bound of a stochastic majority vote classifier and its deterministic counterpart for Equal Opportunity (EOP).}
    \label{fig:plot_EOP}
\end{figure}

\end{document}